\documentclass[letterpaper]{article} 
\usepackage[preprint]{aaai2027}  
\usepackage[hyphens]{url}  
\usepackage{graphicx} 
\usepackage{natbib}  
\usepackage{caption} 
\usepackage{algorithm}
\usepackage{algorithmic}
\usepackage{amsthm}
\usepackage{bm}
\usepackage{amssymb}
\usepackage{amsmath}
\newtheorem{proposition}{Proposition}
\usepackage{newfloat}
\usepackage{listings}
\DeclareCaptionStyle{ruled}{labelfont=normalfont,labelsep=colon,strut=off} 
\floatstyle{ruled}
\newfloat{listing}{tb}{lst}{}
\floatname{listing}{Listing}

\usepackage{booktabs}

\title{Constrained Probabilistic Numerics: \\ Gaussian Process Solvers for Linear PDEs under Linear Inequality Constraints}
\author{
    Boyuan Deng\textsuperscript{\rm 1},
    Kshitiz Upadhyay\textsuperscript{\rm 2},
    Michael Shields\textsuperscript{\rm 1}\corresponding
}
\affiliations{
    \textsuperscript{\rm 1}Department of Civil and Systems Engineering, Johns Hopkins University\\
    \textsuperscript{\rm 2}Department of Aerospace Engineering and Mechanics, University of Minnesota\\

    bdeng8@jh.edu, kshitizu@umn.edu, michael.shields@jhu.edu

}

\begin{document}

\maketitle

\begin{abstract}
    Gaussian process solvers for linear partial differential equations (PDEs) return posteriors rather than point estimates, but they condition only on equality information: the PDE residual, boundary data, and measurements. Many solutions of interest also obey inequality constraints, including nonnegativity, monotonicity, and absence of arbitrage, which equality conditioning cannot enforce. In the absence of inequality constraints, the equality posterior can place probability mass on physically or economically inadmissible solutions. We prove that conditioning a Gaussian process prior jointly on linear equality observations and a finite-dimensional inequality event is equivalent to conditioning sequentially in terms of Borel measures on the solution space. The physics-informed posterior is therefore a valid prior for the inequality update at virtual points. The combined law is non-Gaussian yet computable through truncated-Gaussian moments, and neither the PDE conditioning nor the inequality update requires modification. We evaluated the construction on a Poisson benchmark, option-price surfaces, and air-quality fields.
    Under incomplete boundary data, a monotonicity constraint lowered the error of a two-dimensional Poisson reconstruction by $63\%$. On a real S\&P~500 option chain thinned to $20$ retained quotes, no-arbitrage constraints reduced held-out extrapolation error by $22\%$ and cut the calendar-arbitrage rate from $10\%$ to $1\%$. On four real air-quality snapshots, nonnegativity reduced a Gaussian moment-based sub-zero-tail proxy by $74\%$ at the sparsest monitor count.
    Changes in mean accuracy and constraint diagnostics differed across data regimes. The probability of the noisy virtual constraint event under the equality posterior measures agreement with the specified bounds before the update. The construction turns admissibility from a post-processing step into conditioning information for probabilistic PDE solvers.
\end{abstract}

\begin{links}
    \link{Code}{https://github.com/bydeng01/linpde-gp}
\end{links}

\section{Introduction}

Linear partial differential equations (PDEs) underpin mechanistic models throughout science and engineering, from heat transfer and electromagnetism to wave propagation and the pricing of financial derivatives. The numerical methods that solve them, such as collocation, finite-element, and spectral schemes, return point estimates whose approximation error is seldom expressed in a form that downstream inference can use. Because such solvers are routinely embedded inside larger scientific models, ignoring that error can silently bias the conclusions drawn from them. Probabilistic numerics addresses this by recasting numerical computation as Bayesian inference \citep{hennig2015probabilistic,hennig2022probabilistic, cockayne2019bayesian}. A solver then returns a posterior rather than a single estimate, and its uncertainty propagates coherently through a modeling pipeline. Gaussian processes (GPs) realize this idea naturally for PDEs \citep{rasmussen2006gaussian}. Because a GP is closed under linear operators, observations of a PDE residual are just another linear functional of the unknown solution, an insight exploited since early work on operator observations \citep{graepel2003solving,sarkka2011linear}. Conditioning on the residual at a finite number of collocation points then yields a tractable posterior over the solution \citep{raissi2017machine,chen2021solving}. \citet{pfortner2022physics} make the connection exact. Their physics-informed GP regression conditions on observations made through an arbitrary bounded linear operator, strictly generalizing the method of weighted residuals and unifying collocation, finite-element, and spectral solvers into one framework.

This formalism handles equality information in closed form and with calibrated uncertainty. The PDE residual, the boundary conditions, and noisy measurements all enter as linear equality observations. Yet many solutions of interest also obey inequality constraints that equality conditioning cannot enforce. A concentration field is nonnegative, a transport solution is monotone, and an arbitrage-free option surface is convex in strike and monotone in maturity. Such constraints encode physical or economic admissibility, and in some domains they are theorems rather than modeling choices \citep{carr2005note}. Violating them leaves a posterior that is not merely inaccurate, but inadmissible because its mean and its sample paths can fall outside the feasible set. As a result, downstream computations then inherit mass on impossible states. Inequality constraints of this kind have been addressed separately in the constrained GP regression literature. We specifically adopt the linear-operator construction of \citet{agrell2019gaussianprocesseslinearoperator}, which conditions on bounds for noisy operator evaluations at finitely many virtual locations.

These two capabilities have largely developed independently. Physics-informed GP solvers condition on a linear differential operator but cannot impose an inequality on the solution. Constrained GPs impose inequalities on a data-fit prior but do not condition on a governing PDE. A recent survey catalogues the two as distinct constraint classes \citep{swiler2020survey}. A prominent learning-based alternative, the standard physics-informed neural network (PINN), instead imposes governing laws through soft residual penalties and, when uncertainty-aware, relies on numerical or approximate inference rather than a closed-form posterior \citep{karniadakis2021physics}. We review both approaches in Section~\ref{sec:related_work}. To our knowledge, no prior framework combines exact linear-PDE operator conditioning with linear inequality constraints in a single posterior.

We address this gap with a two-stage construction that assigns distinct roles to equality and inequality information. We first condition a GP prior on exact or noisy linear equality observations from the PDE, boundary conditions, and data. This yields the physics-informed GP posterior of \citet{pfortner2022physics}, which we call the linear PDE Gaussian process (LinPDE-GP) posterior. We then condition this posterior on finite-dimensional linear inequality information at virtual locations using the update of \citet{agrell2019gaussianprocesseslinearoperator}. The resulting posterior is non-Gaussian, but exact identities express the means and covariances of its finite-dimensional marginals through numerically evaluated truncated-Gaussian moments. We study the construction on a Poisson problem, option-price surfaces, and air-quality fields, including real S\&P~500 option data and OpenAQ monitor networks. The experiments separately assess changes in posterior-mean accuracy and constraint diagnostics across data regimes.

This paper makes three contributions spanning theory, computation, and empirical evaluation. First, we fix an explicit Gaussian version of the equality-conditioned law and prove that sequential and joint conditioning yield the same Borel probability measure on the ambient function space (Proposition~\ref{prop:sequential-conditioning}). Second, we develop a practical two-stage solver that augments LinPDE-GP with virtual-point inequality constraints and reuses closed-form operator cross-covariances. Third, we evaluate how adding inequalities changes posterior-mean accuracy and constraint diagnostics as observations become noisier or fewer. We also record $P(C\mid E)$, the probability of the noisy virtual constraint event under the equality posterior.

\section{Related Work}
\label{sec:related_work}

The proposed method merges two existing constrained GP formulations that have advanced largely in parallel: probabilistic solvers that condition Gaussian processes on linear operators, and constrained Gaussian processes that enforce inequalities. We also relate it to physics-informed neural networks (PINNs) and to arbitrage-free option-surface modeling. This section briefly reviews these underlying components.

\paragraph{Probabilistic numerics and GP PDE solvers.}
Probabilistic numerics recasts numerical computation as Bayesian inference \citep{hennig2015probabilistic,cockayne2019bayesian,hennig2022probabilistic}. GPs suit differential equations because they are closed under linear operators. \citet{graepel2003solving} first solved noisy linear operator equations by GP regression. Later, \citet{sarkka2011linear} treated linear operators and stochastic PDEs in the same conditioning framework. Meanwhile, probabilistic meshless methods have been shown to recover symmetric collocation as a posterior mean \citep{cockayne2017probabilistic}. Conditioning on a linear PDE residual at a finite number of points yields a tractable posterior over the solution \citep{raissi2017machine,chen2021solving}. \citet{pfortner2022physics} made this precise by showing that conditioning through an arbitrary bounded linear operator recovers and strictly generalizes the method of weighted residuals, with supporting measure-theoretic guarantees on separable Banach path spaces. The proposed equality stage uses exactly this construction. \citet{chen2022apik} incorporate PDE information into kriging at selected points and actively choose their locations to reduce predictive uncertainty. \citet{li2024parameter} develop PDE-informed Gaussian processes for parameter inference from sparse, noisy observations, including nonlinear PDEs through an augmented formulation. Our construction adds a finite-dimensional inequality update to the Gaussian posterior obtained by conditioning on linear PDE observations. 

\paragraph{Encoding linear constraints in the GP prior.}
A complementary line of research builds linear equality constraints into the prior, so every realization satisfies them by construction. \citet{jidling2017linearly} model the target as a linear transformation of an underlying GP. \citet{lange2018algorithmic,lange2021linearly} then developed this algorithmically, using Gr\"obner bases to parametrize solution sets of linear differential systems. \citet{lange2021linearly} further incorporates boundary conditions. These priors are exact but target only equality constraints and build a bespoke prior per operator, encoding no inequalities. The proposed solver keeps the generic LinPDE-GP recipe and adds inequalities on top.

\paragraph{Inequality and shape-constrained GPs.}
There is also a mature literature incorporating inequality constraints into GP regression. Early work imposed monotonicity through virtual derivative observations \citep{riihimaki2010gaussian}. \citet{golchi2015monotone} incorporate derivative-sign information at selected locations into Bayesian GP emulation, with joint inference over GP parameters, function values, and derivative values. \citet{da2012gaussian} computed constrained predictive moments from the moments of a truncated multivariate Gaussian (the same moment-based route followed by our posterior summaries) and \citet{wang2016estimating} enforced shape constraints at virtual locations with a Gibbs scheme. A finite-dimensional basis approximation can instead encode inequality constraints so that boundedness, monotonicity, and/or convexity hold over the entire domain \citep{maatouk2017gaussian}. Conditioning at a finite number of virtual locations makes the posterior a truncated Gaussian \citep{lopez2018finite,agrell2019gaussianprocesseslinearoperator}. \citet{agrell2019gaussianprocesseslinearoperator} formulated the problem for arbitrary linear operators with exact truncated-Gaussian sampling by minimax tilting \citep{botev2017normal}, which is the approach adopted in this work. This family scales poorly with dimension, motivating the development of dimension-reduction schemes \citep{bachoc2022sequential}. However, all of these methods impose inequalities on a data-fit prior, never on a PDE-conditioned one \citep{swiler2020survey}. Proposition~\ref{prop:sequential-conditioning} (Section~\ref{sec:theory} in this work) enables the application of inequality constraints on PDE-conditioned GP posterior.

\paragraph{Physics-informed neural networks.}
Beyond the GP setting, PINNs are the prevailing tool for embedding governing equations into learning, imposing the PDE as a soft penalty in a non-convex objective \citep{raissi2019physics,karniadakis2021physics}. Later work adds hard constraints through penalty and augmented-Lagrangian schemes \citep{lu2021physics}, and uncertainty through Bayesian PINNs \citep{yang2021b}. Yet the two capabilities we need remain in tension. Soft penalties do not guarantee exact constraint satisfaction, and the uncertainty is approximate rather than a conditional law. Our construction conditions exactly on a finite-dimensional inequality event, defined on noisy observations at virtual points. The inequality therefore enters by truncation at the price of linearity.

\paragraph{Arbitrage-free option-price surfaces.}
Our finance experiment connects to no-arbitrage surface construction, where admissibility is determined by a theorem stating that absence of static arbitrage forces price bounds, monotonicity in maturity, and convexity in strike \citep{carr2005note}. Additionally, the surface solves Dupire's forward equation under a local-volatility model \citep{dupire1994pricing}. Shape-restricted nonparametric option pricing enforces these constraints in estimation \citep{ait2003nonparametric}, and \citet{fengler2009arbitrage} imposes them on smoothing splines. Closer to our setting, GP approaches add uncertainty quantification. \citet{cousin2016kriging} combine market-fit equality with monotonicity and positivity in a kriging term-structure model, and \citet{chataigner2021beyond} propose a GP-based approach that is provably arbitrage-free and yields a local-volatility surface. In both, the equality information is quote interpolation or calibration, and the differential structure enters, if at all, through the generating model rather than as conditioning information. Neither conditions on the Dupire operator itself. Our equality stage does exactly that, and the shape constraints then act on the operator-conditioned posterior.

\paragraph{Positioning.}
Across these strands, physics-informed GP solvers condition on a linear operator but impose no inequalities. Constrained GPs, conversely, impose inequalities on a data-fit or interpolation prior without conditioning on a PDE. The closest precedents combine interpolation or calibration with shape constraints \citep{cousin2016kriging,chataigner2021beyond}, not operator conditioning with shape constraints. We link the two with Proposition~\ref{prop:sequential-conditioning}, which makes the LinPDE-GP posterior a valid prior for Agrell's update, which we show is equivalent to joint conditioning. We are not aware of prior work that combines operator-level PDE conditioning with linear inequality constraints in one posterior.

\section{Problem Formulation}
\label{sec:problem_formulation}

\paragraph{Equality information.}

Let $\mathcal D$ denote the problem domain and let $u\sim\mathcal{GP}(m_0,k_0)$ be a Gaussian process prior with index set $\mathcal D$ on a probability space $(\Omega,\mathcal F,\mathbb P)$. Let $U\subset\mathbb R^{\mathcal D}$ be a real separable reproducing kernel Banach space (RKBS), a Banach space of real-valued functions on $\mathcal D$ on which every point evaluation $\delta_x$, $x\in\mathcal D$, is a continuous linear functional. Point evaluation is therefore itself a bounded linear functional, and both the prediction values and the virtual observations below are built from it. We assume throughout that $u$ satisfies Assumption~1 of \citet{pfortner2022physics}: its sample paths lie in $U$, and the path map $\omega\mapsto u(\cdot,\omega)$ is a $U$-valued Gaussian random element. Every bounded linear image of $u$ in a finite-dimensional space is then a Gaussian random vector (Theorem~1 of \citealp{pfortner2022physics}). The path-space structure behind this assumption is set out in the Theoretical Framework section, where Proposition~\ref{prop:sequential-conditioning} requires it.

A bounded linear operator $L_{\mathrm{eq}}:U\to\mathbb R^{n_{\mathrm{eq}}}$ collects the finite-dimensional equality information.
\begin{equation}
    Y:=L_{\mathrm{eq}}u+\varepsilon_{\mathrm{eq}},\quad
    \varepsilon_{\mathrm{eq}}\sim\mathcal N(0,\Sigma_{\mathrm{eq}}),\quad
    \Sigma_{\mathrm{eq}}\succeq0.
\end{equation}
Here $\varepsilon_{\mathrm{eq}}$ is independent of $u$, and $\Sigma_{\mathrm{eq}}$ permits noise-free and noisy coordinates.
We fix $y\in\mathbb R^{n_{\mathrm{eq}}}$ and write $E$ for the realized equality information $Y=y$. When needed, we partition $Y$ and $y$ into mechanistic blocks $(Y_{\mathrm{phys}},y_{\mathrm{phys}})$ and observational blocks $(Y_{\mathrm{obs}},y_{\mathrm{obs}})$, and denote their realized information by $E_{\mathrm{phys}}$ and $E_{\mathrm{obs}}$. Unless stated otherwise, $E$ denotes both blocks jointly, with any absent block omitted. We label observation-only baselines explicitly and use $+E_{\mathrm{phys}}$ for the addition of the mechanistic block. Under these assumptions, the equality-conditioned process remains Gaussian \citep{pfortner2022physics}. Since $Y=y$ is generally a probability-zero event, the next section fixes the corresponding $U$-valued regular conditional distribution.

\paragraph{Finite-dimensional inequality information.}
We group the inequality constraints into $J$ families, each collecting those that share one linear operator. For each family $j=1,\ldots,J$, choose a scalar-valued linear function-space operator $\mathcal L_j$ and an ordered virtual tuple $X_j^v=(x_{j,1}^v,\ldots,x_{j,m_j}^v)\in\mathcal D^{m_j}$. Let $\delta_{X_j^v}:\mathbb R^{\mathcal D}\to\mathbb R^{m_j}$ be the linear evaluation map $\delta_{X_j^v}v:=\bigl(v(x_{j,1}^v),\ldots,v(x_{j,m_j}^v)\bigr)$, and require $B_j:=\delta_{X_j^v}\circ\mathcal L_j:U\to\mathbb R^{m_j}$ to be bounded. We stack these maps to obtain the finite-dimensional inequality operator $L_{\mathrm{ineq}}:U\to\mathbb R^d$, which is bounded because each $B_j$ is.
\begin{equation}
L_{\mathrm{ineq}}u:=\Bigl((B_1u)^\top,\ldots,(B_Ju)^\top\Bigr)^\top\in\mathbb R^d,\quad
d:=\sum_{j=1}^Jm_j.
\end{equation}
Here $d$ is the number of scalar virtual variables, pooled across the $J$ families. Stacking lets families with different operators, virtual locations, and bound types be written as one box event. Let $a\in(\mathbb R\cup\{-\infty\})^d$ and $b\in(\mathbb R\cup\{+\infty\})^d$ satisfy $a_i<b_i$ componentwise. Let $\varepsilon_v\sim\mathcal N(0,\Sigma_v)$ be independent of $(u,\varepsilon_{\mathrm{eq}})$, with $\Sigma_v\succ0$. The virtual vector and finite-dimensional inequality event are defined by
\begin{equation}
Z:=L_{\mathrm{ineq}}u+\varepsilon_v,
\qquad
C:=\{a\le Z\le b\}.
\end{equation}
The coordinates of $Z$ are the noisy virtual variables. Agrell's construction uses independent virtual noise $\Sigma_v=\sigma_v^2I$ \citep{agrell2019gaussianprocesseslinearoperator}, whereas Proposition~\ref{prop:sequential-conditioning} permits any positive-definite $\Sigma_v$. Because $\operatorname{Cov}(Z\mid Y=y)\succeq\Sigma_v\succ0$ and $a_i<b_i$, $\mathbb P(C\mid Y=y)>0$, so the target law below is well defined. Conditioning on $C$ restricts the noisy virtual vector, so the update is soft at the selected locations and does not enforce the inequalities throughout the domain.

\paragraph{Target law and diagnostic.}
The inferential target is the following conditional law:
\begin{equation}
\mathcal L(u\mid E,C)
:=
\mathcal L(u\mid Y=y,C)
=
\mathcal L(u\mid Y=y,\,a\le Z\le b).
\end{equation}
Although $u\mid E$ is Gaussian, $u\mid E,C$ is generally non-Gaussian. The screening diagnostic is the following conditional probability:
\begin{equation}
\label{eq:constraint-probability}
P(C\mid E)
:=
\mathbb P(C\mid Y=y)
=
\mathbb P(a\le Z\le b\mid Y=y).
\end{equation}
Here $P(C\mid E)$ is shorthand under the fixed regular conditional version, rather than an ordinary probability ratio involving the singleton event $Y=y$. This diagnostic is a $d$-dimensional Gaussian box probability under the equality-conditioned law of $Z$. It depends on $L_{\mathrm{ineq}}$, $a$, $b$, $\Sigma_v$, and the full conditional mean and covariance of $Z$. It therefore changes with the operator family, dimension, virtual locations, marginal scales and correlations. Comparisons are most defensible within a fixed constraint design, rather than as a normalized score across applications.

\section{Theoretical Framework}
\label{sec:theory}

Combining the two methods raises a well-posedness question. The law returned by the PDE solve must be a legitimate prior for the constraint update. The following proposition settles it. Conditioning jointly on the linear equality observations and on the inequality event yields the same Borel measure on $U$ as first forming the Gaussian LinPDE-GP posterior and then applying the finite-dimensional inequality update of \citet{agrell2019gaussianprocesseslinearoperator}. The constrained law is no longer Gaussian, but its finite-dimensional marginals retain Agrell's representation as a compound Gaussian with a truncated Gaussian mean, that is, Gaussian conditionals mixed over a truncated Gaussian virtual-observation variable.

Proposition~\ref{prop:sequential-conditioning} is a statement about measures on $U$, so we set out what Assumption~1 of \citet{pfortner2022physics} provides. The path map $\omega\mapsto u(\cdot,\omega)$ is measurable from $(\Omega,\mathcal F)$ to $\bigl(U,\mathcal B(U)\bigr)$, and its law is a Gaussian measure on $U$, meaning that every continuous linear functional $\ell\in U'$ is univariate Gaussian under it. Given that the sample paths lie in $U$, separability reduces the remaining requirement to Gaussianity of $\omega\mapsto\ell[u(\cdot,\omega)]$ for every $\ell\in U'$. Weak-$*$ sequential density of $\operatorname{span}\{\delta_x\}_{x\in\mathcal D}$ in $U'$ is sufficient for this (Corollary~B.7 of \citealp{pfortner2022physics}). That half of the assumption is therefore often a property of the topology of $U$, whereas path inclusion depends on the regularity of $m_0$ and $k_0$. Section~\ref{sec:method} verifies both for the priors used in our experiments.

\paragraph{Proof roadmap.}
The argument moves between three levels: a measure on the path space $U$, its finite-dimensional images, and a conditioning event defined on those images. It runs in five steps. \emph{Step~1} singles out one concrete Gaussian version of the law of $u$ given $Y=y$, which is needed because that event need not have positive probability. \emph{Step~2} observes that the joint and the sequential construction each pair this equality-conditioned process with independent virtual noise, and that the two pairs have the same joint law. \emph{Step~3} applies one map to both pairs, sending each to a finite-dimensional linear image of the path together with the virtual vector, so the two routes again agree. \emph{Step~4} restricts the virtual vector to $[a,b]$, turning each law into Gaussian conditionals averaged over a truncated Gaussian, the representation the solver of Section~\ref{sec:method} evaluates. \emph{Step~5} lifts agreement of every finite-dimensional linear image to agreement of the two measures on all of $\mathcal B(U)$, which is the assertion of the proposition. Step~1 is carried out immediately below, because the statement of Proposition~\ref{prop:sequential-conditioning} already refers to the law it produces; steps~2 to~5 are the proof.

We now carry out step~1 and fix a version of the equality-conditioned law. Theorem~1 of \citet{pfortner2022physics} states that the process conditioned on $Y=y$ is again Gaussian. We additionally need this conditional law as a measure on $U$, which follows from their results by a short argument that we record here. Let $G:=L_{\mathrm{eq}}k_0L_{\mathrm{eq}}'+\Sigma_{\mathrm{eq}}$ and let $\Phi:\mathbb R^{n_{\mathrm{eq}}}\to U$ be the bounded finite-rank operator $\Phi z:=\sum_{j}(G^{\dagger}z)_{j}\,r_{j}$, built from the representers $r_{j}:x\mapsto L_{\mathrm{eq}}[k_0(x,\cdot)]_{j}$, which lie in $U$ (Proposition~B.13 and Lemma~B.4 of \citealp{pfortner2022physics}). Since $\Sigma_{\mathrm{eq}}$ is only positive semi-definite, $G$ may be singular; the cross-covariance nonetheless vanishes, by the kernel argument of Appendix~\ref{app:measure}. The pair $(u-\Phi Y,\,Y)$ is jointly Gaussian with vanishing cross-covariance, hence independent, so
\begin{equation}
\mathcal L(u\mid Y=y):=\mathcal L\bigl(u-\Phi Y+\Phi y\bigr)
\end{equation}
is a Gaussian measure on $U$, defines a regular conditional distribution of $u$ given $Y$, meaning a family of genuine probability measures on $U$ indexed by the observed value (Appendix~\ref{app:measure}), and has the conditional mean and covariance of Theorem~1 of \citet{pfortner2022physics}. Because $\varepsilon_v$ below is independent of $(u,Y)$, we correspondingly set $\mathcal L\bigl((u,\varepsilon_v)\mid Y=y\bigr):=\mathcal L(u\mid Y=y)\otimes\mathcal N(0,\Sigma_v)$. We use these versions throughout, so no almost-sure qualifications on $y$ are needed.

Step~1 is complete. The law $\mathcal L(u\mid Y=y)$ is now a Gaussian measure on $U$, and the proposition below is stated against this version.

\begin{proposition}[Sequential conditioning with LinPDE-GP and Agrell constraints]
\label{prop:sequential-conditioning}
Adopt the setup and assumptions of Section~\ref{sec:problem_formulation}, in particular $\Sigma_v\succ0$ and $a_i<b_i$ componentwise. Fix $y\in\mathbb R^{n_{\mathrm{eq}}}$ and use the regular conditional version fixed above, under which $\operatorname{Cov}(Z\mid Y=y)\succeq\Sigma_v\succ0$, so $\mathbb P(C\mid Y=y)>0$ and
\begin{equation}
\mathcal L(u\mid Y=y,C)(A)
:=\frac{\mathbb P(u\in A,\,C\mid Y=y)}{\mathbb P(C\mid Y=y)},\quad
A\in\mathcal B(U),
\end{equation}
is well defined. Let $u_{\mathrm{cond}}\sim\mathcal L(u\mid Y=y)$, let $\widetilde\varepsilon_v\sim\mathcal N(0,\Sigma_v)$ be independent of $u_{\mathrm{cond}}$, and define
\begin{equation}
C_{\mathrm{seq}}
:=
\{a\le L_{\mathrm{ineq}}u_{\mathrm{cond}}
+\widetilde\varepsilon_v\le b\}.
\end{equation}
Then
\begin{equation}
\mathcal L(u\mid Y=y,C)
=
\mathcal L(u_{\mathrm{cond}}\mid C_{\mathrm{seq}})
\end{equation}
as Borel probability measures on $U$.
\end{proposition}

\begin{proof}[Proof of Proposition~\ref{prop:sequential-conditioning}]
We carry out steps~2 to~5 of the roadmap; step~1 is the version fixed above. Write $\mu_y:=\mathcal L(u\mid Y=y)$ for that version and $\gamma_v:=\mathcal N(0,\Sigma_v)$.

\emph{Step 2: both routes start from the same law.}
On the joint route, $\varepsilon_v$ is independent of $(u,Y)$, so the version fixed above gives $\mathcal L\bigl((u,\varepsilon_v)\mid Y=y\bigr)=\mu_y\otimes\gamma_v$. This product measure is the joint law of two independent ingredients, the equality-conditioned path and the virtual noise. On the sequential route, $u_{\mathrm{cond}}\sim\mu_y$ and $\widetilde\varepsilon_v\sim\gamma_v$ are independent by construction, so $\mathcal L\bigl(u_{\mathrm{cond}},\widetilde\varepsilon_v\bigr)=\mu_y\otimes\gamma_v$ as well. The two routes therefore begin from one and the same measure on $U\times\mathbb R^d$.

\emph{Step 3: both routes apply the same map.}
Fix $q\in\mathbb N$ and a bounded linear $L_*:U\to\mathbb R^q$, and set
\begin{equation}
T(v,e):=\bigl(L_*v,\;L_{\mathrm{ineq}}v+e\bigr),
\quad
T:U\times\mathbb R^d\to\mathbb R^q\times\mathbb R^d.
\end{equation}
Both $(L_*u,\,Z)$ given $Y=y$ and $\bigl(L_*u_{\mathrm{cond}},\,L_{\mathrm{ineq}}u_{\mathrm{cond}}+\widetilde\varepsilon_v\bigr)$ are images of $\mu_y\otimes\gamma_v$ under $T$: the first by the product version fixed above, the second by the definition of $(u_{\mathrm{cond}},\widetilde\varepsilon_v)$. Equal laws sent through one measurable map have equal images, so the two coincide. We call this the transfer identity. Taking the second marginal gives $\mathbb P(C_{\mathrm{seq}})=\mathbb P(C\mid Y=y)>0$, so $\mathcal L_{\mathrm{seq}}:=\mathcal L(u_{\mathrm{cond}}\mid C_{\mathrm{seq}})$ is well defined.

\emph{Step 4: both routes condition on the same box.}
Keep $L_*$ fixed and write $X:=L_*u$. Being a bounded affine image of the product Gaussian above, $(X,Z)$ is Gaussian given $Y=y$ (Lemma~B.4 of~\citealp{pfortner2022physics}). Because $\operatorname{Cov}(Z\mid Y=y)\succeq\Sigma_v\succ0$, the standard Gaussian conditioning formula defines a version of $\mathcal L(X\mid Y=y,Z=z)$ for every $z$. Its mean is affine in $z$, whereas its covariance is independent of $z$. Step~3 gives the two routes the same joint Gaussian law, so their standard Gaussian conditional versions coincide. Hence no almost-everywhere qualification is needed below. Since $C=\{a\le Z\le b\}\in\sigma(Z)$ and $\mathbb P(C\mid Y=y)>0$, for every Borel $A\subset\mathbb R^q$,
\begin{equation}
\mathbb P(X\in A\mid Y=y,\,C)
=
\frac{\mathbb E\bigl[\mathbf 1_A(X)\,\mathbf 1_{[a,b]}(Z)\,\big|\,Y=y\bigr]}
     {\mathbb P(C\mid Y=y)}.
\end{equation}
We disintegrate the numerator over $Z$, that is, we write it as an average of the Gaussian conditional laws $\mathcal L(X\mid Y=y,Z=z)$ against the law of $Z$ restricted to the box:
\begin{equation}
\mathcal L(X\mid Y=y,C)
=
\int_{[a,b]}
\mathcal L(X\mid Y=y,Z=z)\,
\mathcal L(dz\mid Y=y,C).
\end{equation}
Here $\mathcal L(Z\mid Y=y,C)$ is the Gaussian law $\mathcal L(Z\mid Y=y)$ truncated to $[a,b]$, well defined because $\operatorname{Cov}(Z\mid Y=y)\succeq\Sigma_v\succ0$. The right-hand side is thus Gaussian conditionals averaged over a truncated Gaussian, the compound-Gaussian form that Section~\ref{sec:method} computes. By the transfer identity of step~3, both the Gaussian conditionals and the truncated mixing law coincide with their sequential counterparts, so the mixture equals the law of $L_*u_{\mathrm{cond}}$ conditioned on $C_{\mathrm{seq}}$. Hence, for every bounded linear $L_*:U\to\mathbb R^q$,
\begin{equation}
\mathcal L(L_*u\mid Y=y,C)
=
(L_*)_\#\,\mathcal L_{\mathrm{seq}},
\end{equation}
where $(L_*)_\#\,\mathcal L_{\mathrm{seq}}$ denotes the image of $\mathcal L_{\mathrm{seq}}$ under $L_*$. Since $U$ is an RKBS, $L_*$ may in particular collect finitely many point evaluations, so this is agreement of the two posteriors at every finite set of prediction points.

\emph{Step 5: from finite-dimensional linear images to the path-space law.}
It remains to lift this agreement of finite-dimensional images to $\mathcal B(U)$. A cylinder set asks whether finitely many bounded linear measurements of a path land in a Borel subset of $\mathbb R^q$. Step~4 holds for every bounded linear $L_*$, hence for every finite family $\ell_1,\ldots,\ell_q\in U'$, so it states exactly that the two laws agree on every cylinder set. Intersecting two cylinder sets stacks their functionals, so the cylinder sets are closed under finite intersections, and separability of $U$ makes them generate $\mathcal B(U)$. The sets on which two probability measures agree are closed under complements and countable disjoint unions. Dynkin's $\pi$--$\lambda$ theorem therefore propagates agreement from the generating class to all of $\mathcal B(U)$, and the two laws agree there. Appendix~\ref{app:measure} records these definitions and the theorem.
\end{proof}

Two remarks are in order. First, \citet{agrell2019gaussianprocesseslinearoperator} states the inequality update with independent virtual noise $\Sigma_v=\sigma_v^2I$ and also admits the noiseless limit $\sigma_v=0$; Proposition~\ref{prop:sequential-conditioning} requires only $\Sigma_v\succ0$, and this strict positivity is our standing assumption, which is what renders $\mathbb P(C\mid Y=y)>0$ automatic rather than assumed. Second, the proposition concerns exact laws rather than any particular computation. It establishes that the physics-informed posterior of \citet{pfortner2022physics} is a valid prior for the inequality update of \citet{agrell2019gaussianprocesseslinearoperator}. The next section turns this license into a usable solver.

\section{Proposed Method}
\label{sec:method}

Proposition~\ref{prop:sequential-conditioning} enables a sequential two-update procedure, summarized in Algorithm~\ref{alg:method}. The first update conditions the prior GP on all linear equality information, namely the PDE residual, the boundary data, and any point measurements, and returns the physics-informed Gaussian posterior $u\mid Y=y\sim\mathcal{GP}(m_{\mathrm{phys}},k_{\mathrm{phys}})$ \citep{pfortner2022physics}. The second update conditions this posterior on the inequality event at virtual points. The resulting law is non-Gaussian, so its finite-dimensional marginals follow Agrell's compound-Gaussian construction, and the posterior summaries reduce to truncated-Gaussian moments.

The inequality information enters through the $J$ constraint families of Section~\ref{sec:problem_formulation}, each family $j$ carrying its own function-space operator $\mathcal L_j$ and its own virtual tuple $X_j^v$. This instantiates Proposition~\ref{prop:sequential-conditioning} with $L_{\mathrm{ineq}}$ the stack of the blocks $B_j:=\delta_{X_j^v}\circ\mathcal L_j$. The stack is what collects the $d=\sum_{j}m_j$ scalar virtual quantities of all families into the one vector $Z$, and their bounds into the one box $[a,b]$. Each block must be bounded on $U$. For the squared-exponential priors used here, Assumption~1 of \citet{pfortner2022physics} holds with $U=C^p$-type sample spaces for every $p$ (their Appendix~B.4), so point-evaluated derivatives are continuous and boundedness holds for the two operator classes we instantiate, the identity and first-order partial derivatives. The designs differ in $J$. Nonnegativity of the air-quality field is a single family, and the Poisson conditions $\partial_x u\ge0$ and $\partial_y u\ge0$ are two families on a shared virtual grid. The option-surface price, spread, and calendar conditions are three families on separate grids, with bounds that vary by coordinate. The construction itself accepts any bounded $\mathcal L_j$ with computable cross-covariances, and second-order constraints are left to future work (Appendix~\ref{app:implementation}). Following \citet[Section~3]{agrell2019gaussianprocesseslinearoperator}, the virtual observations carry independent Gaussian noise, so the constraint is enforced softly at the virtual points. Under the stage-one posterior, using the evaluated operator $L_{\mathrm{ineq}}$ defined above, write
\begin{equation}
Z=L_{\mathrm{ineq}}u+\varepsilon_v,
\qquad
\varepsilon_v\sim\mathcal N(0,\Sigma_v),
\end{equation}
with $\Sigma_v=\sigma_v^2I\succ0$ as in Proposition~\ref{prop:sequential-conditioning}. Its mean and covariance are
\begin{equation}
\begin{aligned}
m_Z&=L_{\mathrm{ineq}}m_{\mathrm{phys}}, \\
\Sigma_Z&=L_{\mathrm{ineq}}k_{\mathrm{phys}}L_{\mathrm{ineq}}'+\Sigma_v,
\end{aligned}
\end{equation}
both available in closed form through operator cross-covariances. When $J>1$, $\Sigma_Z$ is block structured, its off-diagonal blocks being the cross-family covariances $(\mathcal L_i k_{\mathrm{phys}}\mathcal L_j')(X_i^v,X_j^v)$. For prediction values $u_*:=u(X^*)$, set
\begin{equation}
A=\operatorname{Cov}(u_*,Z\mid Y=y)\Sigma_Z^{-1},
\end{equation}
and
\begin{equation}
\Sigma_{u_*|Z}
=
\operatorname{Cov}(u_*\mid Y=y)
-
A\operatorname{Cov}(Z,u_*\mid Y=y).
\end{equation}
The constrained predictive law is then represented by
\begin{equation}
\begin{aligned}
z &\sim \mathcal{TN}(m_Z,\Sigma_Z;a,b), \\
u_*\mid z &\sim\mathcal N(\mu_{*|z},\Sigma_{u_*|Z}),\\
\mu_{*|z}&:=m_{\mathrm{phys}}(X^*)+A(z-m_Z).
\end{aligned}
\end{equation}

Posterior summaries follow in closed form from the truncated-Gaussian moments of $z$. Writing $\nu$ and $\Gamma$ for the mean and covariance of $z$, the constrained moments are
\begin{equation}
\begin{aligned}
\mathbb E[u_*\mid Y=y,C]&=m_{\mathrm{phys}}(X^*)+A(\nu-m_Z),\\
\operatorname{Cov}(u_*\mid Y=y,C)&=\Sigma_{u_*|Z}+A\Gamma A^\top,
\end{aligned}
\end{equation}
which specialize the exact compound-moment identities of \citet[Corollary~6]{agrell2019gaussianprocesseslinearoperator} to the two-stage setting. These identities yield the posterior mean and the $1.96$-standard-deviation band. The truncated-Gaussian moments $\nu$ and $\Gamma$ they require have no closed form, and are evaluated by randomized quasi-Monte Carlo, reproducible under a fixed seed but subject to numerical-integration error (Appendix~\ref{app:implementation}). Sample paths, where needed, are drawn with the exact minimax-tilting sampler recommended by \citet[Section~3.2]{agrell2019gaussianprocesseslinearoperator}.

In the implementation, the nonzero stage-one mean is handled by an equivalent shift of the inequality bounds at the virtual points; this and the remaining numerical details are given in Appendix~\ref{app:implementation}.

\begin{algorithm}[tb]
\caption{LinPDE-GP with linear inequality constraints}
\label{alg:method}
\begin{algorithmic}[1]
\STATE {\bfseries Input:} Prior GP $(m_0,k_0)$; equality operator $L_{\rm eq}$, observation $y$, noise covariance $\Sigma_{\rm eq}\succeq0$; inequality operator $L_{\rm ineq}$ assembled from $(\mathcal L_j,X_j^v)_{j=1}^J$, bounds $[a,b]$ (entries in $\mathbb R\cup\{\pm\infty\}$), virtual noise $\Sigma_v=\sigma_v^2I\succ0$; prediction points $X^*$; optional sample count $N$.
\STATE {\bfseries Equality update:}
Condition $u\sim\mathcal{GP}(m_0,k_0)$ on
\[
Y=L_{\rm eq}u+\varepsilon_{\rm eq}=y,
\qquad
\varepsilon_{\rm eq}\sim\mathcal N(0,\Sigma_{\rm eq}),
\]
to obtain $u\mid Y=y\sim\mathcal{GP}(m_{\rm phys},k_{\rm phys})$.
\STATE {\bfseries Virtual inequality variables:}
Compute
\[
m_Z=L_{\rm ineq}m_{\rm phys},
\]
\[
\Sigma_Z=L_{\rm ineq}k_{\rm phys}L_{\rm ineq}'+\Sigma_v.
\]
\STATE {\bfseries Prediction conditioning:}
For $u_*:=u(X^*)$, compute under the stage-one posterior
\[
A=\operatorname{Cov}(u_*,Z\mid Y=y)\Sigma_Z^{-1},
\]
\[
\Sigma_{u_*|Z}
=
\operatorname{Cov}(u_*\mid Y=y)
-
A\operatorname{Cov}(Z,u_*\mid Y=y).
\]
\STATE {\bfseries Moments:}
Compute $\nu=\mathbb E[z]$ and $\Gamma=\operatorname{Cov}(z)$ for $z\sim\mathcal{TN}(m_Z,\Sigma_Z;a,b)$; set
\[
\begin{aligned}
\mu_* &= m_{\rm phys}(X^*)+A(\nu-m_Z), \\
S_* &= \Sigma_{u_*|Z}+A\Gamma A^\top.
\end{aligned}
\]
\IF{$N$ samples are requested}
\STATE Choose $Q$ such that $QQ^\top=\Sigma_{u_*|Z}$.
\FOR{$i=1,\dots,N$}
\STATE Draw $z^{(i)}\sim\mathcal{TN}(m_Z,\Sigma_Z;a,b)$ and $\eta^{(i)}\sim\mathcal N(0,I)$; set
\[
u_*^{(i)}
=
m_{\rm phys}(X^*)
+
A(z^{(i)}-m_Z)
+
Q\eta^{(i)}.
\]
\ENDFOR
\ENDIF
\STATE {\bfseries Return:} Moments $(\mu_*,S_*)$ and, if requested, samples $\{u_*^{(i)}\}_{i=1}^N$.
\end{algorithmic}
\end{algorithm}

\section{Experiments}

In this section, we evaluate whether adding inequality constraints improves reconstruction accuracy, reduces constraint violations, and narrows posterior uncertainty across a set of diverse experiments. The experiments covered a Poisson equation with incomplete observations, option-price surfaces, and nonnegative pollutant concentrations.

\paragraph{Comparisons and replication.}
The primary comparison performed in the following experiments holds the GP prior, equality observations $E$, and evaluation points fixed, then compares the equality posterior $u\mid E$ with the constrained posterior $u\mid(E,C)$. For option prices and synthetic air quality, we also compare observations alone and observations with inequalities but without the PDE. The Poisson ablation additionally applies monotonicity to the prior without any equality observations. In the figures, $E_{\mathrm{phys}}$ or $+E$ denotes equality-type PDE conditioning, and $C$ denotes the inequality constraints.

The synthetic data-assimilation, option-price, and air-quality sweeps each use five noise realizations at fixed, unscrambled Halton observation locations~\citep{halton1960efficiency}. The partial-boundary Poisson reconstruction and its ablation are treated deterministically. The real option-price experiment use 20 seeded quote subsets at each retained-quote count, with a fixed test set. The real air-quality experiment use five seeded train--test splits per snapshot and monitor count.

\paragraph{Accuracy and relative improvement.}
We evaluate the posterior mean $m$ against the known solution or held-out observations by calculating the root-mean-square error for an evaluation set $\mathcal T=\{(\bm{x}_i,y_i)\}_{i=1}^{n_{\mathcal T}}$ as
\[
\operatorname{RMSE}_{\mathcal T}(m)
=
\left[
\frac{1}{n_{\mathcal T}}
\sum_{i=1}^{n_{\mathcal T}}
\bigl(m(\bm{x}_i)-y_i\bigr)^2
\right]^{1/2}.
\]
For any lower-is-better metric $M$ (such as RMSE), we calculate the percentage reduction from adding inequalities as
\begin{equation}
\label{eq:gain}
G_M
=
100\left(
1-\frac{\overline{M}_{E,C}}{\overline{M}_{E}}
\right),
\end{equation}
where the bars denote arithmetic means over noise realizations, quote subsets, or snapshot--split pairs. Thus, $G_M$ compares mean metric values. A negative value of $G_M$ indicates an increase in error or in the relevant diagnostic.

\paragraph{Feasibility and uncertainty.}
We checked constraint violations on each application's evaluation grid $\mathcal G$, including locations outside the virtual-point set. These checks do not guarantee that constraints hold between grid points.

Separately, we recorded $P(C\mid E)$ from Eq.~\eqref{eq:constraint-probability}, the joint probability that all noisy virtual constraints are satisfied before inequality conditioning. We computed this Gaussian box probability numerically using minimax tilting with the full conditional covariance of the virtual variables evaluated over the finite virtual-point set. Comparisons held the virtual locations, operators, bounds, and virtual-noise covariance fixed.

For each posterior being evaluated, $m$ and $s$ denote its mean and marginal standard deviation. Displayed posterior bands span $m(\bm{x})\pm1.96s(\bm{x})$ at each prediction point $\bm{x}$. These are pointwise $95\%$ credible intervals for the Gaussian equality posterior and moment-based dispersion summaries for the non-Gaussian constrained posterior. We report band widths as the average full width over the evaluation grid $\mathcal G$. Appendix~\ref{app:implementation} describes their computation and the virtual-observation noise.

\subsection{Two-dimensional Poisson equation: Monotonicity improves reconstruction with incomplete data}
\label{subsec:poisson2d}
\newcommand{\AssimSeeds}{5}
\newcommand{\AssimNobsFixed}{16}
\newcommand{\AssimSigFixed}{0.1}
\newcommand{\AssimSigLo}{0.02}
\newcommand{\AssimGainSigLo}{6}
\newcommand{\AssimSigHi}{0.2}
\newcommand{\AssimGainSigHi}{26}
\newcommand{\AssimNLo}{10}
\newcommand{\AssimGainNLo}{30}
\newcommand{\AssimNHi}{40}
\newcommand{\AssimGainNHi}{13}
\newcommand{\AssimRmseUnc}{0.095}
\newcommand{\AssimRmseCon}{0.071}
\newcommand{\AssimGain}{25}
\newcommand{\AssimBandUnc}{0.39}
\newcommand{\AssimBandCon}{0.27}
\newcommand{\AssimMaxUnc}{0.33}
\newcommand{\AssimMaxCon}{0.28}
\newcommand{\AssimMinSlope}{-0.47}
\newcommand{\AssimPCE}{0.04}
\newcommand{\AssimGainSigZero}{0.5}
\newcommand{\AssimPCESigZero}{0.998}
\newcommand{\AssimNegCellsSig}{9}
\newcommand{\AssimCellsSig}{20}
\newcommand{\AssimNegCellsN}{7}
\newcommand{\AssimCellsN}{25}

\paragraph{Setup.}
We first consider the Poisson equation, $-\Delta u=f$ on $\mathcal D=(0,1)^2$, with manufactured solution
\[
u_\star(x,y)=\frac{(e^x-1)(e^y-1)}{(e-1)^2}.
\]
This solution increases in both coordinates. We use a zero-mean GP prior with an isotropic squared-exponential covariance, unit amplitude, and lengthscale $0.5$. Equality observations are given by the PDE residual at a $3\times3$ interior grid and $u=0$ at 13 boundary points on the edges $x=0$ and $y=0$. No observations are made along the edges $x=1$ and $y=1$. We impose the inequality constraints $\partial_xu\geq0$ and $\partial_yu\geq0$ on a $4\times4$ virtual grid on $[0.2,0.8]^2$, giving 32 constrained virtual points. Figure~\ref{fig:partialbc}(b) shows the solution and observation locations. We then evaluate reconstruction error on a $30\times30$ grid $\mathcal G$ spanning the full domain. On this grid, we summarize posterior uncertainty by the average full band width
\[
B_{1.96}
=
\frac{1}{|\mathcal G|}
\sum_{\bm{x}\in\mathcal G}2(1.96)s(\bm{x}).
\]

\paragraph{Monotonicity corrects the reconstruction near the unobserved corner.}
Figure~\ref{fig:partialbc}(a) shows the true solution, the solution with equality constraints only, the solution with inequality constraints only, and solution with equality + inequality constraints with $\pm1.96s$ uncertainty bands along the diagonal $x=y$. Conditioning only on the PDE and partial boundary data produces a negative diagonal mean that remains very close to the prior, missing the increase towards $(1,1)$. Adding monotonicity improves the solution and narrows its posterior uncertainty band. The surface plots in Fig.~\ref{fig:partialbc}(c,e) show the corresponding improvement across the full domain. The equality-only mean misses the increasing field, whereas the equality + inequality constrained mean recovers its dominant spatial variation.
\begin{figure*}[!ht]
  \centering
  \includegraphics[width=\textwidth]{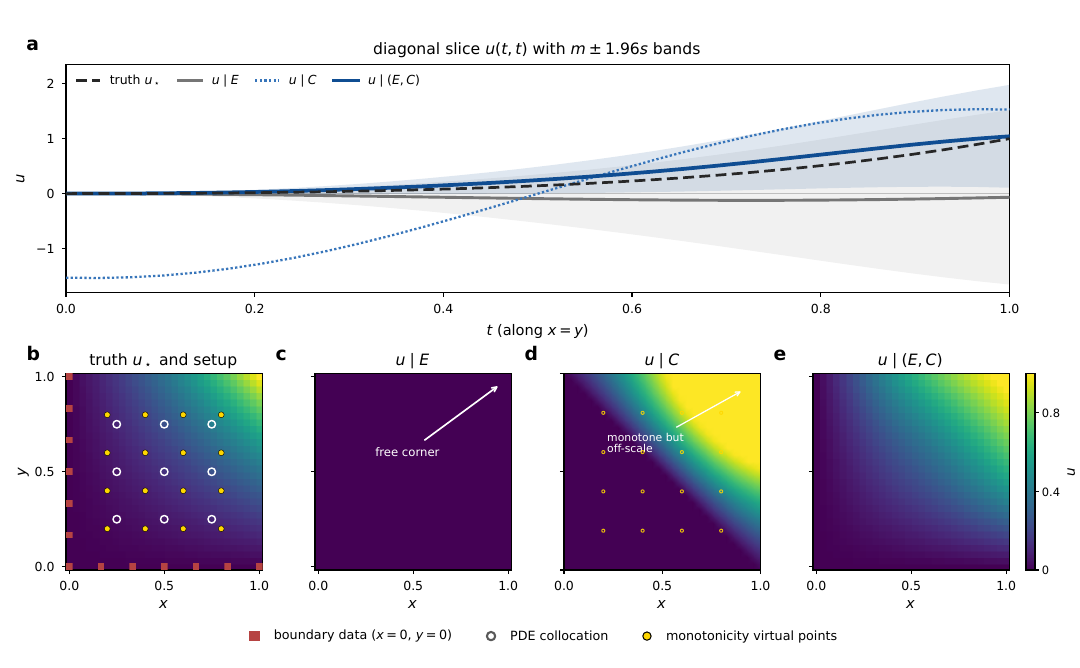}
  \caption{Monotonicity improves Poisson reconstruction with incomplete boundary data. (a) Diagonal posterior means for equality-only ($u|E$), inequality-only ($u|C$), and combined conditioning ($u|(E,C)$), compared with the truth ($u_\star$). Shading shows $m\pm1.96s$ for the equality-only and combined posteriors. (b) Manufactured solution with boundary observations (squares), PDE collocation points (open circles), and monotonicity virtual points (filled circles). (c--e) Posterior means conditioned on the PDE and partial boundary data, monotonicity alone, and both, respectively. Panels (b--e) share a color scale, with values outside $[0,1]$ shown in the endpoint colors.}
  \label{fig:partialbc}
\end{figure*}

Figures~\ref{fig:ablation}(a,b) show that the largest error reduction occurs near the unobserved corner when monotonicity constraints are added to the equality constraints. Table~\ref{tab:poisson2d_partialBC} provides a quantitative comparison of the solutions with equality constraints only and those that also include monotonicity constraints. By adding the inequality constraints, the absolute error at $(1,1)$ reduces from $1.07$ to $0.04$, and the full-grid RMSE reduces from $0.33$ to $0.12$, a $63\%$ reduction. The spatial error patterns are illustrated in Figure~\ref{fig:ablation}(a,b). The average posterior band width decreases from $0.93$ to $0.52$, with the corresponding half-widths shown in Figure~\ref{fig:ablation}(c,d). Residual uncertainty remains largest at $(1,1)$, and an overshoot of up to $0.24$ remains along the unobserved edges. The virtual constraint event has probability $P(C\mid E)=0.11$ before adding monotonicity.

\begin{figure*}[!ht]
  \centering
  \includegraphics[width=0.92\textwidth]{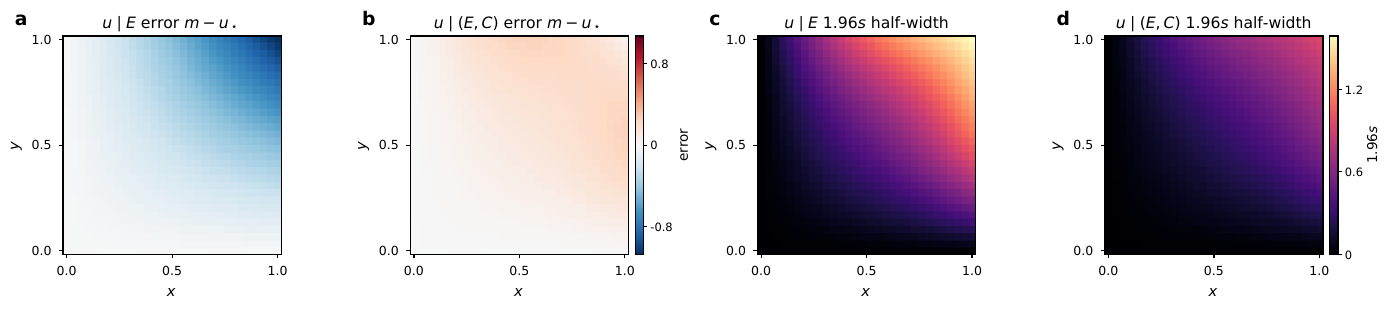}
  \caption{Signed reconstruction errors and posterior uncertainty for the Poisson problem with incomplete boundary data. Equality data $E$ comprise the PDE and partial boundary data; $C$ denotes the monotonicity constraints. (a,b) Signed errors $m-u_\star$ of the equality-only and combined posterior means, respectively. (c,d) Corresponding posterior half-widths $1.96s$, where $s$ is the posterior standard deviation, showing uncertainty magnitude alone. Each pair shares a color scale. Uncertainty is largest at the unobserved corner $(1,1)$ and is reduced by conditioning on monotonicity.}
  \label{fig:ablation}
\end{figure*}

\begin{table}[!ht]
  \centering
  \setlength{\tabcolsep}{6pt}
  {\small
  \begin{tabular*}{\columnwidth}{@{\extracolsep{\fill}}p{0.40\columnwidth}ccc@{}}
    \toprule
     & Equality only
     & With monotonicity
     & Reduction \\
     & $u\mid E$ & $u\mid(E,C)$ & \\
    \midrule
    RMSE & 0.328 & 0.122 & $63\%$ \\
    Maximum absolute error & 1.067 & 0.242 & $77\%$ \\
    Mean full band width & 0.930 & 0.516 & $45\%$ \\
    Absolute error at $(1,1)$ & 1.067 & 0.043 & $96\%$ \\
    \bottomrule
  \end{tabular*}
  }
  \caption{Poisson reconstruction errors and posterior dispersion with partial boundary data. Both predictions use the same prior, PDE collocation points, and boundary observations. Band width refers to the full width of $m\pm1.96s$, averaged over the evaluation grid.}
  \label{tab:poisson2d_partialBC}
\end{table}

\paragraph{Monotonicity alone does not correctly determine the solution magnitude.}
Figure~\ref{fig:partialbc} again compares three reconstructions using the same geometry and reference solution (Figure~\ref{fig:partialbc}(b)). Recall that the PDE and boundary data alone missed the increasing nature of the field (Figure~\ref{fig:partialbc}(c)). Meanwhile, enforcing monotonicity alone captures the increasing trend but fails to recover its magnitude as illustrated in Figure~\ref{fig:partialbc}(d)). It is, indeed, necessary to include both constraints to produce the most accurate solution.
This is further illustrated by the diagonal comparison in Figure~\ref{fig:partialbc}(a). The monotonicity-only mean ranged from approximately $-1.53$ to $1.53$, whereas the true solution ranged from $0$ to $1$. Its RMSE was $0.70$, compared with $0.33$ for the equality-only reconstruction and $0.12$ for the combination. Figure~\ref{fig:ablation} also shows that the combined posterior had the smallest average band width and free-corner error. Monotonicity supplied useful shape information, but the PDE and boundary observations were needed to anchor the solution. 

Finally, Figure~\ref{fig:poisson_ablation_metrics} shows the RMSE, band width, and absolute error at the unobserved corner. Applying inequality constraints only provides the highest overall error and largest band width. However, monotonicity alone improves the solution at the (1,1) corner compared to equality constraints alone. Not surprisingly, the model with both constraints performs the best in all measures.
\begin{figure*}[!ht]
  \centering
  \includegraphics[width=0.6\textwidth]{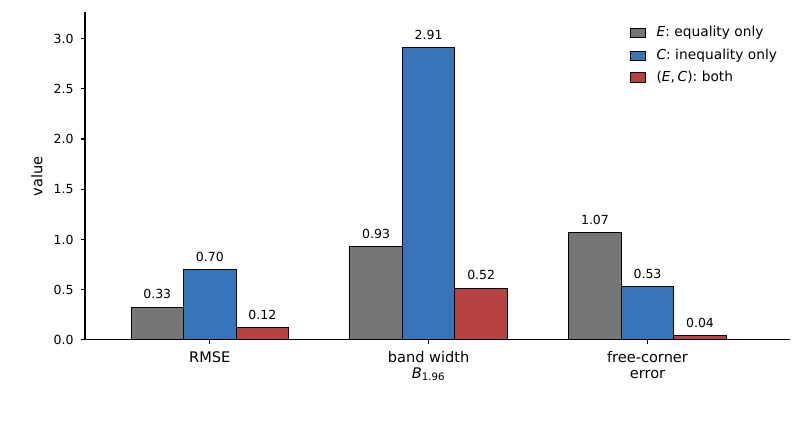}
  \caption{Poisson reconstruction with PDE and partial boundary data ($E$), monotonicity alone ($C$), and both $(E,C)$. Bars show full-grid RMSE, mean full width of the $m\pm1.96s$ bands, and absolute error at $(1,1)$.}
  \label{fig:poisson_ablation_metrics}
\end{figure*}

\paragraph{Monotonicity provides larger relative gains with noisy or sparse observations.}
We next replaced the boundary data with $n_{\mathrm{obs}}$ noisy observations at fixed Halton points on $[0,0.85]^2$, retaining the PDE collocation and virtual grids. Observation errors were independent $\mathcal N(0,\sigma_{\mathrm{obs}}^2)$ variables, and the corner $(1,1)$ remained outside the observed region.

Figure~\ref{fig:assim}(a) shows that adding monotonicity reduced mean reconstruction error more for higher observation noise. With $n_{\mathrm{obs}}=\AssimNobsFixed$, the relative RMSE reduction increased from $\AssimGainSigLo\%$ at $\sigma_{\mathrm{obs}}=\AssimSigLo$ to $\AssimGainSigHi\%$ at $\sigma_{\mathrm{obs}}=\AssimSigHi$. With exact observations, adding monotonicity reduced RMSE by only $\AssimGainSigZero\%$, and the equality posterior already assigned probability $P(C|E)=\AssimPCESigZero$ to the virtual constraint event.
Figure~\ref{fig:assim}(b) further shows that adding monotonicity constraints yields a larger improvement when the number of observations is smaller, in general.  
At fixed $\sigma_{\mathrm{obs}}=\AssimSigFixed$, the RMSE reduction was maximized at $\AssimGainNLo\%$ for $n_{\mathrm{obs}}=10$ and decreases to $\AssimGainNHi\%$ for $n_{\mathrm{obs}}=40$. At the sparsest setting of $n_{\mathrm{obs}}=6$, the reduction was slightly lower at approximately $28\%$.

At the reference setting, $n_{\mathrm{obs}}=\AssimNobsFixed$ and $\sigma_{\mathrm{obs}}=\AssimSigFixed$, adding monotonicity reduced mean RMSE from $\AssimRmseUnc$ to $\AssimRmseCon$ and average band width from $\AssimBandUnc$ to $\AssimBandCon$. Prior to adding monotonicity constraints, the model had a low probability of satisfying these constraints given by $P(C\mid E)=\AssimPCE$. Mean RMSE decreased at every tested setting, although individual realizations worsened in $\AssimNegCellsSig$ of $\AssimCellsSig$ positive-noise cases and $\AssimNegCellsN$ of $\AssimCellsN$ observation-count cases.

\begin{figure}[!ht]
  \centering
  \includegraphics[width=\textwidth]{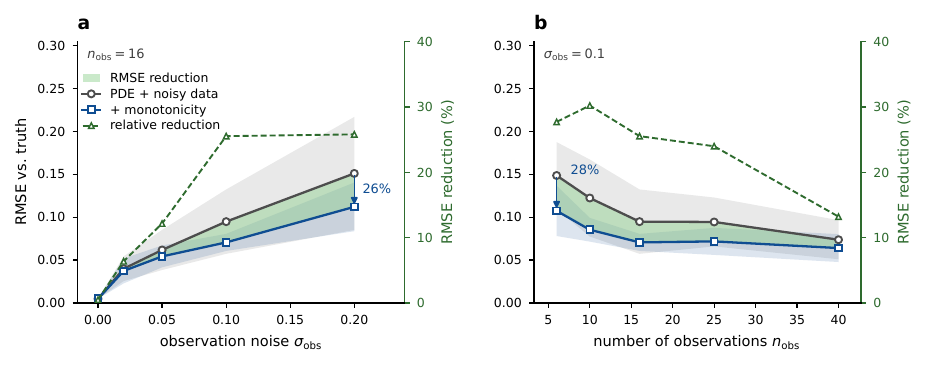}
  \caption{Effect of observation noise magnitude $\sigma_{\mathrm{obs}}$ and number of observations $n_{\mathrm{obs}}$ on the benefit of adding monotonicity constraints. (a) Noise sweep at $n_{\mathrm{obs}}=\AssimNobsFixed$. (b) Observation-count sweep at $\sigma_{\mathrm{obs}}=\AssimSigFixed$. Solid curves show full-grid RMSE against the true solution $u_\star$, averaged over \AssimSeeds{} noise realizations. Ribbons show $\pm1$ standard deviation across realizations. Green shading gives the absolute RMSE reduction. Dashed curves report the relative reduction from Eq.~\eqref{eq:gain} on the right-hand axes. Monotonicity gives larger relative improvements with noisy observations and at the lower observation counts.}
  \label{fig:assim}
\end{figure}

\subsection{Arbitrage-constrained option-price surfaces}
\newcommand{\BsSeeds}{5}
\newcommand{\BsNobsFixed}{16}
\newcommand{\BsSigFixed}{0.01}
\newcommand{\BsSigLo}{0.002}
\newcommand{\BsGainSigLo}{54}
\newcommand{\BsSigHi}{0.04}
\newcommand{\BsGainSigHi}{63}
\newcommand{\BsNLo}{10}
\newcommand{\BsGainNLo}{56}
\newcommand{\BsNHi}{60}
\newcommand{\BsGainNHi}{42}
\newcommand{\BsRmseUnc}{0.047}
\newcommand{\BsRmseCon}{0.019}
\newcommand{\BsGain}{59}
\newcommand{\BsBandUnc}{0.140}
\newcommand{\BsBandCon}{0.053}
\newcommand{\BsViolSpreadUnc}{21}
\newcommand{\BsViolCalUnc}{15}
\newcommand{\BsViolBflyUnc}{19}
\newcommand{\BsViolSpreadCon}{15}
\newcommand{\BsViolCalCon}{13}
\newcommand{\BsViolBflyCon}{17}
\newcommand{\BsPCE}{0.04}

\newcommand{\BsSpxSeeds}{20}
\newcommand{\BsSpxNoise}{0.02}
\newcommand{\BsSpxNobs}{20}
\newcommand{\BsSpxNLo}{12}
\newcommand{\BsSpxGainNLo}{23}
\newcommand{\BsSpxNHi}{300}
\newcommand{\BsSpxGainNHi}{9}
\newcommand{\BsSpxRmseUnc}{0.032}
\newcommand{\BsSpxRmseCon}{0.025}
\newcommand{\BsSpxGain}{22}
\newcommand{\BsSpxGainMed}{14}
\newcommand{\BsSpxPosFrac}{70}
\newcommand{\BsSpxGainMedNLo}{16}
\newcommand{\BsSpxPosFracNLo}{60}
\newcommand{\BsSpxNDip}{60}
\newcommand{\BsSpxGainDip}{11}
\newcommand{\BsSpxGainMedDip}{0}
\newcommand{\BsSpxPosFracDip}{50}
\newcommand{\BsSpxBandUnc}{0.128}
\newcommand{\BsSpxBandCon}{0.051}
\newcommand{\BsSpxViolCalUnc}{10}
\newcommand{\BsSpxViolCalCon}{1}
\newcommand{\BsSpxPCE}{0.01}
\newcommand{\BsSpxDenseRmseQuotes}{0.075}
\newcommand{\BsSpxDenseRmseUnc}{0.027}
\newcommand{\BsSpxDenseRmseCon}{0.027}
\newcommand{\BsSpxDenseGain}{0}
\newcommand{\BsSpxDenseNobs}{3658}

\label{sec:experiments_black_scholes}

We next tested whether price and derivative bounds improve predictions outside the observed region of an option-price surface. For a European call with strike $K$, maturity $T$, and price $C_{\mathrm{call}}(K,T)$, we inferred the normalized surface $c(x,T)=C_{\mathrm{call}}(K,T)/S_0$, where $x=\ln(K/S_0)$ and $S_0$ is the spot price. We used the price bounds, call-spread bounds, and calendar monotonicity conditions
\[
\begin{gathered}
\max(1-e^{x-rT},0)\leq c\leq1,\quad
-e^{x-rT}\leq\partial_x c\leq0,
\qquad
\partial_T c\geq0,
\end{gathered}
\]
where $r$ is the risk-free rate~\citep{fengler2009arbitrage,cousin2016kriging}. Under constant volatility $\sigma$, the equality model was Dupire's forward equation in log-moneyness coordinates~\citep{dupire1994pricing},
\[
\mathcal Lc
:=
\partial_Tc
-\frac{\sigma^2}{2}\bigl(\partial_{xx}c-\partial_xc\bigr)
+r\,\partial_xc
=0.
\]

\paragraph{Setup.}
The synthetic reference surface followed Black--Scholes with $\sigma=0.2$, $r=0.02$, and $S_0=1$, on $x\in[\ln0.7,\ln1.3]$ and $T\in[0.05,1]$ years. We used a zero-mean squared-exponential GP with lengthscales $0.45$ in $x$ and $0.6$ in $T$. PDE conditioning used a $6\times5$ interior collocation grid. Sixteen training quotes occupied the central region $|x|\leq\ln1.05$, $T\leq0.55$, while 40 held-out quotes occupied the wings $|x|>\ln1.15$ and the long-maturity region $T>0.8$ as illustrated with the true solution in Figure~\ref{fig:bs_synth}(a). Both training and held-out synthetic quotes included Gaussian noise with standard deviation $\sigma_{\mathrm{obs}}$.

Each inequality family used a uniform $3\times3$ grid of virtual points, with $T\in[0.35,0.85]$ years and $x$ spanning $[\ln0.85,\ln1.15]$ for price bounds, $[\ln0.88,\ln1.12]$ for call-spread bounds, and $[\ln0.95,\ln1.15]$ for calendar monotonicity, giving 27 constrained virtual variables, also shown in Figure~\ref{fig:bs_synth}.
We assess constraint violations using central differences of the posterior mean $m$ on the evaluation grid. The upper spread-bound, calendar, and butterfly-convexity conditions are $\partial_xm\leq0$, $\partial_Tm\geq0$, and $\partial_{xx}m-\partial_xm\geq0$, respectively. Violation rates use the numerical tolerance specified in Appendix~\ref{app:constraint-diagnostics}. Butterfly convexity was evaluated but not imposed.

\begin{figure*}[!ht]
  \centering
  \includegraphics[width=\textwidth]{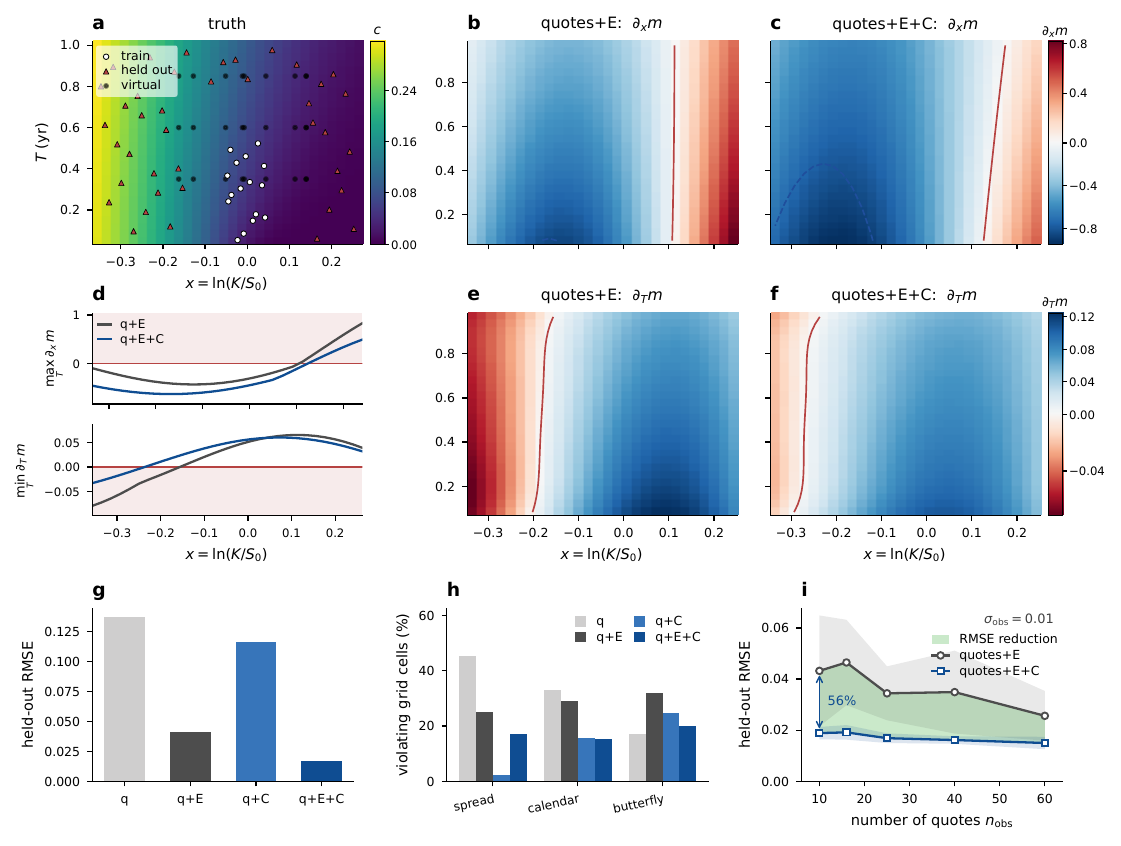}
  \caption{Price and derivative constraints improve extrapolation of a synthetic option-price surface. Derivatives in panels (b)--(f) are computed from the posterior means by central differences. (a) Black--Scholes reference surface with training quotes, held-out quotes, and virtual points. (b,c) Log-moneyness derivatives $\partial_x m$ before and after adding inequality constraints to the PDE-conditioned model. Positive values violate the upper spread bound. The dashed blue contour marks the lower bound. (d) Maximum $\partial_x m$ and minimum $\partial_T m$ over maturities, plotted against log-moneyness. Adding inequality constraints reduces both the magnitude and spatial extent of violations. Red shaded regions indicate infeasible values. (e,f) Maturity derivatives $\partial_T m$ before and after adding inequality constraints. Negative values shown in red violate calendar monotonicity. (g) Held-out RMSE for quotes alone, quotes with the PDE, quotes with inequality constraints alone, and quotes with both. (h) Constraint-violation rates of the posterior means, including butterfly convexity, which was not imposed. Panels (b)--(h) show one noise realization. (i) Mean held-out RMSE versus quote count at $\sigma_{\mathrm{obs}}=\BsSigFixed$, with $\pm1$ standard-deviation ribbons across \BsSeeds{} realizations. Green shading gives the absolute RMSE reduction; the annotation gives the relative reduction.}
  \label{fig:bs_synth}
\end{figure*}

\paragraph{Inequality constraints improve extrapolation and reduce derivative violations.}
Figures~\ref{fig:bs_synth}(b,c) show that adding inequality constraints reduced positive values of $\partial_xm$ in the right wing, where the PDE-conditioned prediction incorrectly increased with strike. Figures~\ref{fig:bs_synth}(e,f) show the corresponding reduction in negative values of $\partial_Tm$ in the left wing, where predicted prices incorrectly decreased with maturity. Figure~\ref{fig:bs_synth}(d) quantifies both changes across maturities. In the displayed realization, the largest upper spread-bound violation fell from $0.83$ to $0.49$, and the largest calendar violation fell from $0.079$ to $0.033$. The violating regions also became smaller.

Figure~\ref{fig:bs_synth}(g) shows that combining the PDE and inequality constraints yields the lowest held-out RMSE among the four predictions conditions. Across \BsSeeds{} noise realizations at $n_{\mathrm{obs}}=\BsNobsFixed$ and $\sigma_{\mathrm{obs}}=\BsSigFixed$, adding the inequality constraints reduced mean held-out RMSE from $\BsRmseUnc$ to $\BsRmseCon$ ($\BsGain\%$). The average posterior band width decreased from $\BsBandUnc$ to $\BsBandCon$, and the virtual constraint probability before the update was $P(C\mid E)=\BsPCE$.

Figure~\ref{fig:bs_synth}(h) compares violation rates for the displayed realization. Averaged over all \BsSeeds{} realizations, adding the inequality constraints reduced the upper spread-bound violation rate from $\BsViolSpreadUnc\%$ to $\BsViolSpreadCon\%$ and the calendar violation rate from $\BsViolCalUnc\%$ to $\BsViolCalCon\%$. Butterfly-convexity violations also decreased, from $\BsViolBflyUnc\%$ to $\BsViolBflyCon\%$, despite convexity not being imposed. These reductions improved the extrapolated surface, while residual violations remained on the evaluation grid.

\paragraph{Inequality constraints reduce the accuracy loss from using fewer quotes.}
Figure~\ref{fig:bs_synth}(i) shows that adding the inequality constraints reduced held-out RMSE by $\BsGainNLo\%$ with $n_{\mathrm{obs}}=\BsNLo$ quotes and by $\BsGainNHi\%$ with $n_{\mathrm{obs}}=\BsNHi$ quotes, at fixed $\sigma_{\mathrm{obs}}=\BsSigFixed$. The constrained RMSE was $0.019$ and $0.015$, respectively, compared with $0.043$ and $0.026$ for the PDE-conditioned predictions. Thus, adding the constraints substantially reduced the accuracy loss associated with fewer observations.

A separate noise sweep, not shown in Fig.~\ref{fig:bs_synth}, held $n_{\mathrm{obs}}=\BsNobsFixed$ fixed. Adding the inequalities reduced mean held-out RMSE by $\BsGainSigLo\%$ at $\sigma_{\mathrm{obs}}=\BsSigLo$ and by $\BsGainSigHi\%$ at $\sigma_{\mathrm{obs}}=\BsSigHi$. Both predictions became less accurate as noise increased, but the relative reduction was larger at the noisier endpoint.

\paragraph{The constraints also improve predictions from sparse real quotes.}
We applied the same construction to a snapshot of an S\&P~500 option chain (Appendix~\ref{app:data}). Put--call-parity regression supplied the normalization and rates, and at-the-money implied volatilities supplied a single volatility estimate. We imposed the constant-volatility PDE with $\sigma_{\mathrm{pde}}>0$ to allow departures from this approximate model.

To test sparse, noisy observations, we randomly retained subsets of the training quotes and added Gaussian noise with standard deviation $\BsSpxNoise$ in normalized price units. The held-out wing and long-maturity quotes remained fixed. Figure~\ref{fig:bs_spx}(a) shows that adding the inequality constraints reduced mean held-out RMSE across \BsSpxSeeds{} subsets regardless of the number of retained quotes. The RMSE reduction was $\BsSpxGainNHi\%$ with $\BsSpxNHi$ retained quotes and increased to $\BsSpxGainNLo\%$ with only $\BsSpxNLo$ quotes. At $n_{\mathrm{obs}}=\BsSpxNobs$, RMSE fell from $\BsSpxRmseUnc$ to $\BsSpxRmseCon$ ($\BsSpxGain\%$), with $P(C\mid E)=\BsSpxPCE$ before the update.

Figure~\ref{fig:bs_spx}(b) shows that adding the inequality constraints reduced calendar violations at every tested quote count. At $n_{\mathrm{obs}}=\BsSpxNobs$, the violation rate fell from $\BsSpxViolCalUnc\%$ to $\BsSpxViolCalCon\%$. Improvements in accuracy varied across subsets. The median RMSE reduction at $n_{\mathrm{obs}}=\BsSpxNobs$ was $\BsSpxGainMed\%$, and $\BsSpxPosFrac\%$ of subsets improved. At $n_{\mathrm{obs}}=\BsSpxNDip$, the aggregate reduction was $\BsSpxGainDip\%$, but the median was approximately zero and half the subsets improved. This is consistent with the observation that inequality constraints are more effective in reducing errors when fewer observations are available.

For the dense control case, we used all $\BsSpxDenseNobs$ training quotes with their bid--ask-derived noise level and no added perturbations. PDE conditioning reduced held-out RMSE from $\BsSpxDenseRmseQuotes$ to $\BsSpxDenseRmseUnc$, and adding the inequalities left it at $\BsSpxDenseRmseCon$. This is shown by the gray box mark in Figure~\ref{fig:bs_spx}(a). The constraints therefore added little to the densely observed surface, while improving predictions in the sparse, noisy reconstructions.

\begin{figure}[!ht]
  \centering
  \includegraphics[width=0.90\columnwidth]{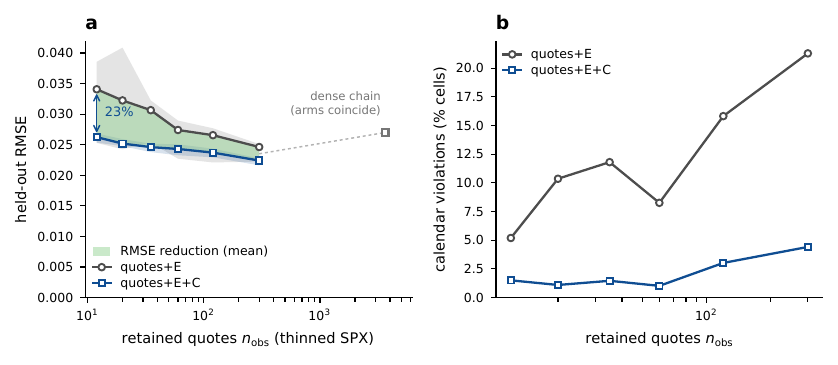}
  \caption{Effects of inequality constraints on an S\&P~500 option chain with experimentally thinned and perturbed training quotes. (a) Mean held-out RMSE across \BsSpxSeeds{} subsets at each retained-quote count. Gray shading shows the $25$th--$75$th percentiles. Adding the inequality constraints reduces mean error throughout the thinning experiment. The separate dense control (square gray point) uses all $\BsSpxDenseNobs$ training quotes at their bid--ask-derived noise level, where the two RMSE values coincide. (b) Calendar-violation rates of the posterior mean surfaces, averaged over the same subsets, decrease after adding the inequality constraints at every tested quote count.}
  \label{fig:bs_spx}
\end{figure}

\subsection{Nonnegative reconstruction of an air-quality field}
\label{subsec:air_quality}
\newcommand{\AqSeeds}{5}
\newcommand{\AqNobsFixed}{14}
\newcommand{\AqSigFixed}{0.1}
\newcommand{\AqSigLo}{0.05}
\newcommand{\AqGainSigLo}{7}
\newcommand{\AqSigHi}{0.3}
\newcommand{\AqGainSigHi}{-43}
\newcommand{\AqTailRedSigLo}{23}
\newcommand{\AqTailRedSigHi}{56}
\newcommand{\AqNLo}{8}
\newcommand{\AqGainNLo}{-19}
\newcommand{\AqNHi}{28}
\newcommand{\AqGainNHi}{32}
\newcommand{\AqTailRedNLo}{41}
\newcommand{\AqTailRedNHi}{28}
\newcommand{\AqRmseE}{0.076}
\newcommand{\AqRmseEC}{0.088}
\newcommand{\AqGain}{-16}
\newcommand{\AqBandE}{0.35}
\newcommand{\AqBandEC}{0.30}
\newcommand{\AqSubzeroE}{0.160}
\newcommand{\AqSubzeroEC}{0.107}
\newcommand{\AqTailRed}{33}
\newcommand{\AqNegFrac}{13}
\newcommand{\AqMinVal}{-0.26}
\newcommand{\AqPCE}{0.09}

\newcommand{\AqRealSnaps}{4}
\newcommand{\AqRealSnapNoun}{snapshots}
\newcommand{\AqRealSnapsFull}{2}
\newcommand{\AqRealFolds}{5}
\newcommand{\AqRealCells}{90}
\newcommand{\AqRealMonDense}{22}
\newcommand{\AqRealMonSparse}{6}
\newcommand{\AqRealTailRedDense}{47}
\newcommand{\AqRealTailRedSparse}{74}
\newcommand{\AqRealTailRedDenseErr}{6}
\newcommand{\AqRealTailRedSparseErr}{3}
\newcommand{\AqRealClipDense}{1}
\newcommand{\AqRealClipSparse}{7}
\newcommand{\AqRealClipMax}{7}
\newcommand{\AqRealHeldGainDense}{7}
\newcommand{\AqRealHeldGainSparse}{3}
\newcommand{\AqRealPCEDense}{0.78}
\newcommand{\AqRealPCESparse}{0.47}
\newcommand{\AqRealMonCommon}{10}
\newcommand{\AqRealTailCommon}{67}
\newcommand{\AqRealSnapsMono}{4}
\newcommand{\AqRealCrpsGainSparse}{4}
\newcommand{\AqRealCrpsGainDense}{7}
\newcommand{\AqRealCrpsGainAll}{3}
\newcommand{\AqRealNllGainAll}{-1}
\newcommand{\AqRealIscGainAll}{3}
\newcommand{\AqRealCovE}{0.62}
\newcommand{\AqRealCovEC}{0.60}
\newcommand{\AqRealCovClip}{0.62}

In the final example, we tested nonnegativity constraints on a pollutant concentration field $c(x,y)$ governed by advection and diffusion. Unlike the derivative constraints above, this experiment constrains the field values directly.

\paragraph{Setup and feasibility diagnostic.}
The synthetic field satisfies the steady advection--diffusion equation
\[
Lc=-D\Delta c+\mathbf v\cdot\nabla c+kc=S
\]
on $\mathcal D=(0,1)^2$, with constant diffusivity $D$, wind velocity $\mathbf v$, decay rate $k$, and a nonnegative Gaussian emission source $S$. A second-order finite-difference solution supplied the reference plume. We conditioned a zero-mean GP on the PDE residual at interior collocation points and noisy monitor readings, then added $c\geq0$ constraints on a $3\times3$ virtual grid. Figure~\ref{fig:air_quality}(a) shows the reference field and the monitor, held-out, and virtual locations. Synthetic held-out RMSE was evaluated against the noise-free reference concentrations.
\begin{figure}[!ht]
  \centering
  \includegraphics[width=\columnwidth]{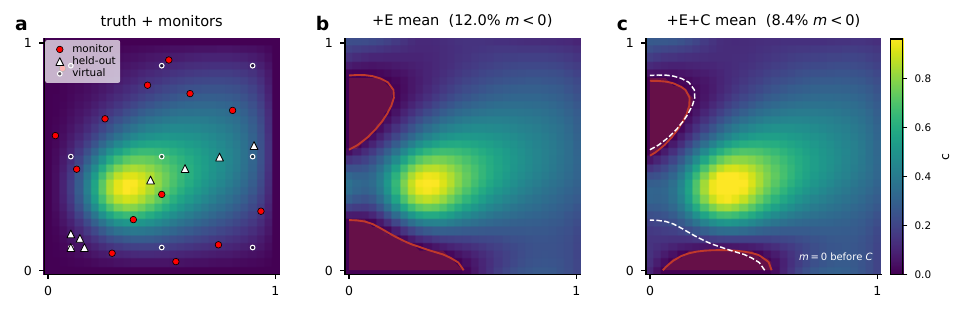}
  \caption{Nonnegativity reduces negative posterior means in a synthetic plume reconstruction. Results use $n_{\mathrm{obs}}=\AqNobsFixed$, $\sigma_{\mathrm{obs}}=\AqSigFixed$, and one noise realization. (a) Finite-difference reference field with monitor, held-out, and virtual locations. (b) Posterior mean conditioned on monitors and the PDE, with negative values overlaid in red. (c) Adding nonnegativity reduces the negative region. The dashed white contour marks its boundary before the update. Table~\ref{tab:air_quality_arms} reports metrics for this realization.}
  \label{fig:air_quality}
\end{figure}

In addition to the fraction of negative posterior means, we calculate the grid-averaged marginal probability of a negative concentration as
\begin{equation}
\label{eq:subzero-proxy}
Q_0
=
\frac{1}{|\mathcal G|}
\sum_{\bm x\in\mathcal G}
\Phi_{\mathrm{std}}\!\left(-\frac{m(\bm x)}{s(\bm x)}\right),
\end{equation}
where $\Phi_{\mathrm{std}}$ is the standard normal distribution function. For the non-Gaussian constrained posterior, it is a Gaussian moment-based sub-zero-tail proxy.

\paragraph{The nonnegativity constraint reduces negative predictions.}
Figure~\ref{fig:air_quality}(b) shows that conditioning on the PDE and sparse monitors produced negative posterior means in the clean-air region. Figure~\ref{fig:air_quality}(c) shows that adding nonnegativity reduced this region of constraint violation. At $n_{\mathrm{obs}}=\AqNobsFixed$ and $\sigma_{\mathrm{obs}}=\AqSigFixed$, the fraction of negative grid values fell from $12.0\%$ to $8.4\%$ in the displayed realization. The constrained mean was nonnegative at every virtual point, with residual negative values elsewhere on the grid implying that a denser grid of constraint points will likely continue to improve the solution.

Table~\ref{tab:air_quality_arms} compares the four prediction cases for this realization. Adding nonnegativity to the PDE-conditioned model reduced held-out RMSE from $0.142$ to $0.075$ and average band width from $0.355$ to $0.305$. Monitors with nonnegativity constraints but without the PDE constraints gave the largest RMSE of $0.314$, compared with $0.086$ for monitors alone, suggesting that applying only the non-negativity constraint can be detrimental to the solution.

Across \AqSeeds{} noise realizations at the same setting, adding nonnegativity constraints reduced $Q_0$ from $\AqSubzeroE$ to $\AqSubzeroEC$ ($\AqTailRed\%$) and narrowed the average band width from $\AqBandE$ to $\AqBandEC$. Mean held-out RMSE nevertheless increased from $\AqRmseE$ to $\AqRmseEC$. The accuracy improvement in the illustrated realization therefore did not persist after averaging over observation noise. The virtual constraint probability before the update was $P(C\mid E)=\AqPCE$.

\begin{table}[!ht]
  \centering
  \setlength{\tabcolsep}{6pt}
  {\small
  \begin{tabular*}{\columnwidth}{@{\extracolsep{\fill}}p{0.50\columnwidth}ccc@{}}
    \toprule
    Conditioning
    & \shortstack{Held-out RMSE}
    & \shortstack{Fraction $m<0$}
    & \shortstack{Band width} \\
    \midrule
    Monitors
      & 0.086 & 0.102 & 1.602 \\
    Monitors + PDE
      & 0.142 & 0.120 & 0.355 \\
    Monitors + nonnegativity
      & 0.314 & \textbf{0.078} & 1.282 \\
    Monitors + PDE + nonnegativity
      & \textbf{0.075} & 0.084 & \textbf{0.305} \\
    \bottomrule
  \end{tabular*}
  }
  \caption{Synthetic air-quality reconstruction for the noise realization shown in Figure~\ref{fig:air_quality}. Negative-mean fractions and full $m\pm1.96s$ band widths are evaluated over the grid. Bold values mark the minimum in each column.}
  \label{tab:air_quality_arms}
\end{table}

\paragraph{Tail-proxy reduction improves with greater noise but at the expense of mean accuracy.}
Figure~\ref{fig:air_quality_sweep} shows tail-proxy reduction and held-out RMSE separately at fixed $n_{\mathrm{obs}}=\AqNobsFixed$ for increasing noise. Increasing $\sigma_{\mathrm{obs}}$ from $\AqSigLo$ to $\AqSigHi$ improved the $Q_0$ reduction from $\AqTailRedSigLo\%$ to $\AqTailRedSigHi\%$. Over the same range, the held-out RMSE reduction changed from $\AqGainSigLo\%$ to $\AqGainSigHi\%$. The negative reduction indicates that the constraint increased error for larger noise values.
A separate monitor-count sweep, not shown in Figure~\ref{fig:air_quality_sweep}, held $\sigma_{\mathrm{obs}}=\AqSigFixed$ fixed. Reducing the monitor count from $\AqNHi$ to $\AqNLo$ increased the $Q_0$ reduction from $\AqTailRedNHi\%$ to $\AqTailRedNLo\%$, while the RMSE reduction decreased from $\AqGainNHi\%$ to $\AqGainNLo\%$. Thus, stronger reductions in the sub-zero-tail proxy coexisted with poorer mean predictions. With this zero-mean prior, the upward shift produced by the nonnegativity constraint can overshoot concentrations near zero in clean-air regions.

\begin{figure}[!ht]
  \centering
  \includegraphics[width=0.5\columnwidth]{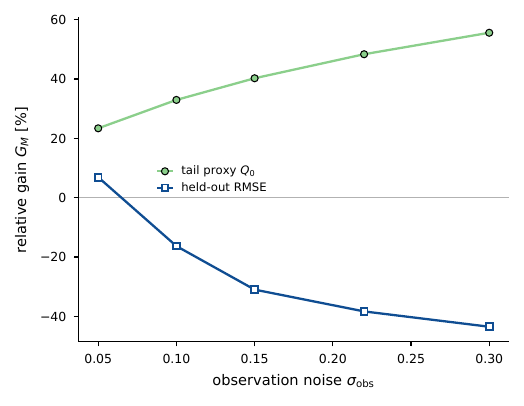}
  \caption{Observation noise separates improvements in the sub-zero tail diagnostic from improvements in mean accuracy. Curves show percentage gains in the sub-zero-tail proxy $Q_0$ and held-out RMSE after adding the nonnegativity constraint, for fixed $n_{\mathrm{obs}}=\AqNobsFixed$. Gains compare means over \AqSeeds{} noise realizations using Eq.~\eqref{eq:gain}. The gains in $Q_0$ grow with noise, whereas the RMSE reduction becomes negative, indicating reduced accuracy. The horizontal line marks no change relative to the PDE-conditioned prediction.}
  \label{fig:air_quality_sweep}
\end{figure}

\paragraph{Real monitor networks show lower tail proxies and mixed accuracy changes.}
We repeated the analysis on \AqRealSnaps{} OpenAQ daytime snapshots of $\mathrm{PM}_{2.5}$ and $\mathrm{NO}_2$ from Maryland and New York (Appendix~\ref{app:data}). Each snapshot used its own ERA5-derived wind direction and a source-free transport model with $S=0$. We evaluated \AqRealFolds{} seeded train--test splits per snapshot and sampled monitor count, giving \AqRealCells{} evaluations. Here, monitor count denotes the sampled network before stations were withheld for scoring.

The dark line in Figure~\ref{fig:air_quality_real}(a) shows that adding nonnegativity resulted in larger $Q_0$ gaints as the networks were thinned by sampling fewer monitors. Across all \AqRealSnaps{} snapshots, the reduction increased from $\AqRealTailCommon\%$ at \AqRealMonCommon{} monitors to $(\AqRealTailRedSparse\pm\AqRealTailRedSparseErr)\%$ at \AqRealMonSparse{} monitors, where the uncertainty is a bootstrap standard error. Over the same range, mean $P(C\mid E)$ decreased from $0.55$ to $0.47$.
The lighter curve in Figure~\ref{fig:air_quality_real}(a) follows the two Maryland $\mathrm{PM}_{2.5}$ snapshots across the full monitor-count range. In general the $Q_0$ gains increase as the number of sampled monitors decreases, except at $N_{\mathrm{mon}}=14$, which is likely a result of sample noise.
For comparison, we set negative values of the equality-posterior mean to zero and left the variance unchanged. The dashed curve shows that this correction reduced $Q_0$ much less than adding the nonnegativity constraint.

Figure~\ref{fig:air_quality_real}(b) shows that changes in held-out RMSE were small and could be either positive or negative. Across all four snapshots, thinning from ten to six monitors changed the RMSE reduction from approximately $-1.5\%$ to $2.6\%$. Adding nonnegativity therefore consistently reduced the moment-based tail diagnostic, while its effect on held-out mean accuracy depended on the sampled network.

\begin{figure}[!ht]
  \centering
  \includegraphics[width=\columnwidth]{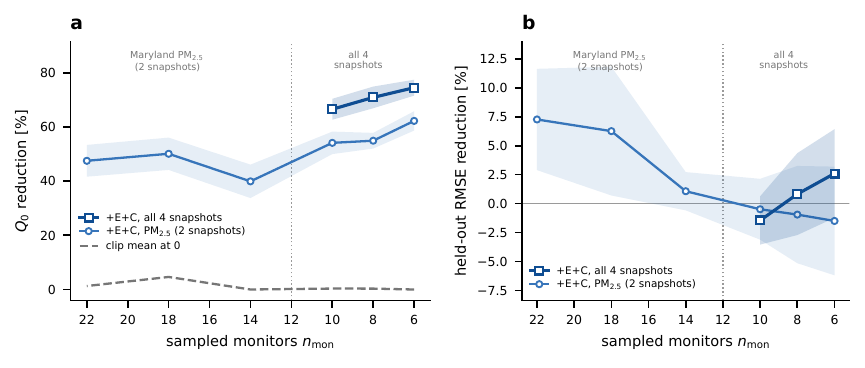}
  \caption{Effects of nonnegativity constraints on real OpenAQ reconstructions. (a) Gains in the Gaussian moment-based sub-zero-tail proxy $Q_0$. The dashed baseline clips the equality-posterior mean at zero and retains its variance, using the two snapshots that span the full monitor-count range. (b) Held-out RMSE reduction, which can be positive or negative. The dark blue curves pool all \AqRealSnaps{} snapshots over their common monitor counts. Light blue curves pool the two Maryland $\mathrm{PM}_{2.5}$ snapshots over the full range. Monitor counts decrease from left to right and include stations subsequently held out for scoring. Gains compare pooled means. Shading denotes bootstrap standard errors from 2,000 resamples of snapshot--split pairs at each count.}
  \label{fig:air_quality_real}
\end{figure}

\section{Discussion and Limitations}

The proposed general framework for GP solutions to linear PDEs and the test cases examined demonstrate how PDE solution constraints can be combined with inequality information to improve accuracy, ensure solution validity, and reduce uncertainty. Nonetheless, these results show that gains are not always consistent and there are areas for improvement. The following presents a brief summary of key takeaways from this work and avenues for future research.

\paragraph{Inequality information can complement equality constraints, particularly when observations are limited and/or noisy.}
The major takeaway from this work is that equality and inequality constraints can be combined in a unified framework and that those components can work together to improve probabilistic solutions over those that apply only one constraint or the other. This is especially true when observations are sparse and/or noisy as demonstrated in each of the examples. 
The Poisson and option-price experiments show that shape constraints can improve predictions over those that have only PDE constraints. In the Poisson ablation, satisfying only the PDE and limited, potentially noisy boundary constraints was insufficient to accurately identify the solution. Meanwhile, applying monotonicity constraints alone ensured the solution maintained an increasing trend but did not appropriately scale the solution magnitude. Combining monotonicity with the PDE and boundary observations reduced reconstruction error and prediction uncertainty in the presence of limited and/or noisy data. The synthetic option-price comparison likewise favored combining the PDE and inequalities over using either with the limited quotes alone. By contrast, the noise-free Poisson case and densely quoted option surface showed that inequality constraints provide little additional benefit when observations are abundant and noise free. These comparisons support treating inequality constraints as an additional source of information whose value depends on the number and quality of observations.

\paragraph{Accuracy, constraint violation, and dispersion diagnostics are distinct outcomes that are not uniformly improved by combining constraints.}
We found that, while combining equality and inequality constraints generally tends to improve solution diagnostics, there may be disagreement between different diagnostics. For example, in the reported air-quality averages, nonnegativity reduced the sub-zero-tail proxy $Q_0$ while sometimes increasing held-out RMSE. With a prior centered at zero, upward shifts can artificially inflate concentrations in clean-air regions where values near zero are appropriate. The synthetic option-price sweeps also contained settings where held-out RMSE decreased but derivative-violation rates increased (Appendix~\ref{app:constraint-diagnostics}). Accuracy and constraint violations diagnostics therefore require separate evaluation, and improving one cannot be expected to improve the other. Posterior band width provides yet another diagnostic. A narrower band indicates lower solution uncertainty but should not be mistaken as a measure of solution accuracy.

\paragraph{Evaluating $P(C|E)$ can help determine when inequality constraints may be needed.}
The probability $P(C\mid E)$ describes how often the solution satisfies the specified inequality constraints under the equality posterior. It is a Gaussian box probability that can be evaluated before the inequality update \citep{botev2017normal}. Its value depends on the equality posterior, constraint operators, bounds, virtual locations, and virtual noise. Within a fixed design, a value near one means that the constraint event is already nearly guaranteed to be satisfied, whereas a small value implies that inequality constraints may be needed because the equality-alone solution has a high likelihood of violating the constraints. This may arise from limited observations, noisy observations, proximity to a constraint boundary, or model misspecification. Importantly, the probability $P(C|E)$ alone neither distinguishes these explanations nor predicts the resulting accuracy gained by adding inequality constraints. It is therefore a diagnostic of the proposed constraint event, with predictive consequences assessed through the reconstruction and held-out comparisons. In several cases studied here, $P(C|E)$ was shown to be low and therefore motivated supplementing the solution with inequality constraints.

\paragraph{The construction combines existing constrained-GP capabilities.}
The method combines operator-level Gaussian conditioning \citep{pfortner2022physics} with conditioning on noisy virtual inequalities \citep{agrell2019gaussianprocesseslinearoperator}. Proposition~\ref{prop:sequential-conditioning} establishes that the equality posterior is a valid prior for the second update, and we prove that sequential application of equality then inequality constraints is probabilistically equivalent to joint constraint application. This approach complements methods that encode linear equalities in the prior \citep{jidling2017linearly,lange2021linearly} or represent inequality-constrained fields through finite bases \citep{maatouk2017gaussian}. Here, the chosen GP prior and the Gaussian equality update are retained, while the inequality event updates both posterior means and covariances.

\paragraph{The constraints are finite and soft.}
The construction conditions a finite vector of noisy operator evaluations. Positive virtual noise makes the restrictions soft even at the selected locations, and feasibility elsewhere is assessed on the evaluation grid. Virtual-point placement and noise are therefore part of the model specification. The equivalence of sequential and joint conditioning is exact, whereas the truncated-Gaussian moments are evaluated numerically. Their computation becomes expensive as the virtual dimension increases, motivating more efficient integration and constraint-selection methods \citep{bachoc2022sequential}. The implementation used only identity and first-order derivative constraints. Extending it to second-order operators is straightforward and would allow, for example, butterfly convexity to be imposed directly.


\section{Conclusion}

We established a posterior formulation that combines linear PDE observations with linear inequality information. Sequential Gaussian equality conditioning followed by conditioning on a finite virtual constraint event yields the same law as joint conditioning. The resulting solver retains operator cross-covariances and computes the constrained posterior moments through truncated-Gaussian calculations.

The experiments showed that monotonicity, option-price bounds, and nonnegativity can contribute information beyond the PDE and limited noisy observations. Their effects on reconstruction error, constraint diagnostics, and posterior dispersion differed across settings. The probability $P(C\mid E)$ measures agreement with the specified virtual constraints before the update, while predictive benefits are evaluated against reference solutions or held-out data. With soft constraints at a finite number of virtual locations, the construction incorporates physical and economic admissibility directly into probabilistic PDE inference.

\section*{Acknowledgments}
This material is based upon work supported by the National Science Foundation under Grant Nos. 2331295 and 2623467.

\appendix

\section{Measure-Theoretic Background}
\label{app:measure}

The proof of Proposition~\ref{prop:sequential-conditioning} uses three standard notions, collected here in the form in which the proof uses them. Nothing in this appendix is new.

\paragraph{Regular conditional distributions.}
A regular conditional distribution of $u$ given $Y$ is a map $y\mapsto\mu_y$ such that $\mu_y$ is a Borel probability measure on $U$ for every $y$, such that $y\mapsto\mu_y(A)$ is measurable for every $A\in\mathcal B(U)$, and such that $\mathbb P(u\in A,\,Y\in B)=\int_B\mu_y(A)\,\mathcal L(Y)(dy)$ for all $A\in\mathcal B(U)$ and $B\in\mathcal B(\mathbb R^{n_{\mathrm{eq}}})$. Such a map exists here. Any two such maps can differ only on an $\mathcal L(Y)$-null set. At an atom $y$ with $\mathbb P(Y=y)>0$, the defining identity gives $\mu_y(A)=\mathbb P(u\in A,\,Y=y)/\mathbb P(Y=y)$ for every $A\in\mathcal B(U)$, so the conditional law is uniquely determined there. Step~1 fixes one explicit Gaussian version for every $y$, including values with $\mathbb P(Y=y)=0$, so Proposition~\ref{prop:sequential-conditioning} has a definite meaning for any prescribed $y$.

\paragraph{Image measures.}
For a measurable $T:E\to F$ and a probability measure $\mu$ on $E$, the image measure $T_{\#}\mu$ is defined by $(T_{\#}\mu)(A):=\mu\bigl(T^{-1}(A)\bigr)$ for $A\in\mathcal B(F)$; equivalently, $X\sim\mu$ implies $T(X)\sim T_{\#}\mu$. The only consequence used in steps~2 and~3 is that equal laws have equal images: $\mathcal L(X_1)=\mathcal L(X_2)$ implies $\mathcal L\bigl(T(X_1)\bigr)=\mathcal L\bigl(T(X_2)\bigr)$.

\paragraph{Cylinder sets and the $\pi$--$\lambda$ theorem.}
For $q\in\mathbb N$, functionals $\ell_1,\ldots,\ell_q\in U'$ and $B\in\mathcal B(\mathbb R^q)$, the associated cylinder set is
\begin{equation}
C_{\ell,B}:=\bigl\{v\in U:\ (\ell_1(v),\ldots,\ell_q(v))\in B\bigr\},
\end{equation}
and $\mathcal C_U$ denotes the collection of all such sets, taken over every $q$, every finite family of functionals and every $B$. Each records finitely many bounded linear measurements of a path. Intersecting two cylinder sets stacks their functionals and intersects the corresponding finite-dimensional conditions, so the cylinder sets are closed under finite intersections, that is, they form a $\pi$-system. A $\lambda$-system is a collection of subsets that contains $U$ and is closed under complements and countable disjoint unions; for any two probability measures on $U$, the class of Borel sets on which they agree is one. Dynkin's $\pi$--$\lambda$ theorem states that a $\lambda$-system containing a $\pi$-system contains the $\sigma$-algebra that the $\pi$-system generates. When $U$ is separable, the Hahn--Banach theorem supplies a countable family in $U'$ that norms $U$. Closed balls are then countable intersections of cylinder sets, every open set is a countable union of balls, and so $\sigma(\mathcal C_U)=\mathcal B(U)$. Step~5 combines these three facts.

\paragraph{A singular equality Gram matrix.}
The construction of step~1 uses the pseudoinverse $G^{\dagger}$, and $G$ need not be invertible because $\Sigma_{\mathrm{eq}}$ is only positive semi-definite. The vanishing of the cross-covariance is then not immediate. Write $K_{uY}c:=\sum_j c_j r_j$ for the cross-covariance of $u$ with $Y$, so that $\Phi=K_{uY}G^{\dagger}$ and $\operatorname{Cov}(u-\Phi Y,Y)=K_{uY}(I-G^{\dagger}G)$, where $I-G^{\dagger}G$ is the orthogonal projection onto $\ker G$. That residual vanishes: for $c\in\ker G$ we have $\operatorname{Var}(c^{\top}Y)=c^{\top}Gc=0$, so $c^{\top}Y$ is almost surely constant and $K_{uY}c=\operatorname{Cov}(u,c^{\top}Y)=0$. Hence $K_{uY}=K_{uY}G^{\dagger}G$. The cross-covariance in step~1 therefore vanishes whether or not $G$ has full rank.

\section{Implementation Details}
\label{app:implementation}

This appendix records the implementation choices omitted from the main text for brevity. Each one is needed to reproduce the reported numbers. We group them by operator assembly, constraint scope, the inequality update, the reference solution, and posterior summaries.

\paragraph{Differential operator assembly.}
The air-quality generator $Lc=-D\,\Delta c+\mathbf v\cdot\nabla c+k\,c$ combines second-, first-, and zeroth-order terms. We assemble it as a weighted sum of three named \texttt{linpde-gp} operators: a \texttt{WeightedLaplacian}, a \texttt{DirectionalDerivative}, and an identity term. For the squared-exponential kernel (\texttt{ExpQuad}), each named operator carries a closed-form cross-covariance. The library therefore evaluates $L_{\bm{x}}L_{\bm{x}'}k_0(\bm{x},\bm{x}')$ in closed form, where the two operators act on the first and second kernel arguments, respectively. We checked this chain against a nested automatic-differentiation reference, and the two agreed to within $10^{-6}$. The assembly route matters in practice, because a single merged differential operator falls back to bare first-order partials, which \texttt{ExpQuad} does not dispatch in closed form.

\paragraph{Constraint scope and the inequality wrapper.}
The option-surface experiment applies price bounds, two-sided call-spread bounds, and calendar monotonicity through the \texttt{gp\_constr} component of the wrapper. Each family is evaluated on a $3\times3$ virtual grid, giving $d=27$ scalar virtual variables in the inequality update. The wrapper supports identity and first-order derivative constraints. Butterfly convexity requires second-order derivatives and is evaluated only as a finite-difference diagnostic.

\paragraph{Grid-based option-price diagnostics.}
\label{app:constraint-diagnostics}
We estimated derivatives of the posterior mean $m$ by central differences on the evaluation grid. The upper spread-bound, calendar, and butterfly violation rates were the fractions of valid finite-difference locations satisfying, respectively,
\[
\partial_x m>\tau,
\qquad
\partial_T m<-\tau,
\qquad
\partial_{xx}m-\partial_xm<-\tau,
\]
with $\tau=10^{-4}$. The spread diagnostic counts violations of the upper derivative bound only.

The quote-count sweep included settings where lower prediction error coexisted with higher violation rates. At $n_{\mathrm{obs}}=\BsNLo$ and $\sigma_{\mathrm{obs}}=\BsSigFixed$, adding the inequalities reduced mean held-out RMSE by $\BsGainNLo\%$. Over the same \BsSeeds{} noise realizations, the mean calendar-violation rate increased from $6.0\%$ to $12.9\%$, and the mean butterfly-violation rate increased from $9.9\%$ to $16.7\%$. These grid-based diagnostics measure changes beyond the finite virtual constraint event.

\paragraph{Inequality update: virtual noise and mean handling.}
All experiments use virtual-observation noise $\Sigma_v=\sigma_v^2 I$ with $\sigma_v^2=10^{-4}$. When several derivative blocks share virtual locations, $\Sigma_Z$ can become numerically singular; $\sigma_v^2$ is then raised geometrically until the factorization succeeds, which only enlarges the virtual noise (a final eigenvalue floor guards the pathological case). The inequality update acts on the centered process $u-m_{\mathrm{phys}}$, with the bounds shifted by $L_{\mathrm{ineq}}m_{\mathrm{phys}}$; the shift is exact pathwise by linearity. It is evaluated by second-order central differences with step $10^{-4}$, a bound perturbation of order $10^{-8}$, negligible against $\sigma_v^2$. All covariance terms use closed-form operator cross-covariances.

\paragraph{Reference solution.}
The synthetic plume in the air-quality experiment comes from a second-order finite-difference solve of $Lc=S$. A manufactured-solution test gives an observed convergence rate of approximately $2$, which confirms the discretization order.

\paragraph{Posterior summaries.}
Pointwise posterior bands span the posterior mean plus or minus $1.96$ marginal posterior standard deviations. For the Gaussian equality posterior, this is the exact marginal $95\%$ credible interval. The inequality-constrained posterior is non-Gaussian, so its corresponding band is a Gaussian-equivalent dispersion summary rather than a $95\%$ credible interval. We compute the constrained standard deviation from the truncated-Gaussian covariance rather than estimating it from posterior samples. The constrained mean and covariance follow from exact compound-moment identities \citep[Corollary~6]{agrell2019gaussianprocesseslinearoperator}; only the truncated-Gaussian moments $\nu$ and $\Gamma$ that enter them lack a closed form, and we evaluate these by seeded randomized quasi-Monte Carlo~\citep{genz2009computation}. The seed removes run-to-run variability from the reported summaries, but numerical-integration error remains.

\paragraph{Computational cost.}
The two-stage solve has two principal costs. The equality update is closed-form Gaussian conditioning, at the usual cubic cost in the number of equality observations. The inequality update evaluates the truncated-Gaussian moments $\nu$ and $\Gamma$, the mean and covariance of the $d$ virtual variables, and this evaluation dominates the constrained solve. Its cost is governed by the constraint dimension $d$, not by the number of prediction points. The covariance $\Gamma$ alone has $d(d+1)/2$ entries, so the cost grows at least quadratically in $d$. The three experiments used $d=32$ for the Poisson problem, $d=27$ for the option surface, and $d=9$ for the air-quality field. These arise from two $16$-point derivative blocks, three $3\times3$ blocks, and one $3\times3$ identity block, respectively. The dimensions are modest; the same computation scales poorly as the virtual set grows, the recognized bottleneck of this family \citep{bachoc2022sequential}.

\section{Data and Preprocessing}
\label{app:data}

The two real-data experiments draw on public sources. The S\&P~500 option chain is a single-day snapshot of the \texttt{\textasciicircum SPX} listing from Yahoo Finance.\footnote{\url{https://finance.yahoo.com}, retrieved with the \texttt{yfinance} client.} The pollutant fields are measurements from the OpenAQ~v3 REST API,\footnote{OpenAQ, \url{https://openaq.org}; v3 REST API documented at \url{https://docs.openaq.org}, access requires a free API key.} and the transport winds are sampled hourly at 10~m from the Open-Meteo historical archive,\footnote{\url{https://open-meteo.com}.} which serves ECMWF ERA5 reanalysis. The snapshot files and every fetch and preprocessing script are released in the accompanying code repository (see the Code link in the main text). The subsections below record the provenance of each source and the preprocessing that turns it into the quantities the experiments condition on.

\subsection{S\&P~500 option-price surface}
\label{app:data_spx}

\paragraph{Snapshot.}
The chain was captured on 2026-06-13 and holds $18{,}352$ raw quotes, $8{,}142$ calls and $10{,}210$ puts, across $49$ listed expiries reaching to December~2031. 
Both option types are retained at this stage, because the forward and the discount rate are recovered from call and put pairs.

\paragraph{Quote filtering.}
A quote is dropped when its bid is nonpositive or its relative bid--ask spread $(\mathrm{ask}-\mathrm{bid})/\mathrm{mid}$ exceeds $25\%$, and maturities are restricted to $T\in[0.05,1]$~years. Filtering leaves $11{,}382$ quotes over $30$ expiries, spanning roughly $19$ to $342$ calendar days. Each surviving quote enters at its mid price.

\paragraph{Forward, rate, and spot.}
For each expiry $T$, we fit the put--call-parity regression $C_{\mathrm{call}}(K,T)-P_{\mathrm{put}}(K,T)=a_T+b_TK$ by least squares on strikes quoted for both option types, requiring at least three such strikes. Here, $P_{\mathrm{put}}(K,T)$ denotes the put price. The slope satisfies $b_T=-e^{-r_TT}$, giving the expiry-specific rate $r_T=-\ln(-b_T)/T$, and the intercept gives the forward $F_T=-a_T/b_T$. The across-expiry medians define the spot $S_0=\operatorname{median}_T\!\left(F_Te^{-r_TT}\right)\approx7.4\times10^{3}$
index points and the single discount rate $r=\operatorname{median}_T r_T\approx7\%$
on this snapshot. Prices and strikes are then normalized to $c=C_{\mathrm{call}}/S_0$ and $x=\ln(K/S_0)$.

\paragraph{Volatility and quote noise.}
The constant-coefficient Dupire operator needs one volatility. We set it to the average of the at-the-money implied volatilities, taken over quotes with log-moneyness $|x|<0.03$ and each recovered by Brent inversion of the normalized Black--Scholes call price. The observation noise on quote $i$ is half its bid--ask spread in normalized units, $\sigma_i=0.5(\mathrm{ask}-\mathrm{bid})/S_0$. The model uses a single observation-noise level, the median $\sigma_i$, about $2\times10^{-4}$ on this chain.

\paragraph{Evaluation split.}
The surface is fit and scored on the retained calls whose log-moneyness lies in $[\ln 0.7,\ln 1.3]$, $4{,}629$ quotes in all. 
The held-out set is the extrapolation block, log-moneyness $|x|>\ln 1.15$ or maturity $T>0.8$~years, which holds $971$ quotes; the remaining $\BsSpxDenseNobs$ bulk quotes form the training pool. This split is fixed for every arm and every subsample, so the score target never moves.

\paragraph{Dense control and thinning.}
The dense control used all $\BsSpxDenseNobs$ bulk quotes at their original mid prices. Its observation-noise level was the median normalized bid--ask half-spread, denoted by $\sigma_{\mathrm{chain}}$. To test sparse, noisy observations, we randomly retained
$n_{\mathrm{obs}}\in\{12,20,35,60,120,300\}$
quotes from this pool, using $\BsSpxSeeds$ independently seeded subsets per count. For each subset, we set
\[
\sigma_{\mathrm{obs}}
=
\max\{\sigma_{\mathrm{chain}},\BsSpxNoise\}
=
\BsSpxNoise.
\]
We added independent $\mathcal N(0,\sigma_{\mathrm{obs}}^2)$ noise to the retained training mid prices and used the same $\sigma_{\mathrm{obs}}$ in the observation likelihood. All predictions were scored against the same $971$ held-out market mid prices, without added noise. The dense control retained the original training mid prices and the chain-derived noise level.

\subsection{OpenAQ ground-monitor networks}
\label{app:data_openaq}

\paragraph{Source and aggregation.}
Each snapshot is one pollutant, either $\mathrm{PM}_{2.5}$ or $\mathrm{NO}_2$, pulled from the OpenAQ~v3 hourly endpoint inside a city or state bounding box. The pull spans a single daytime window of up to six mid-day hours (UTC $16$:$00$ to $22$:$00$, local afternoon on the US East coast). The hourly readings at a station are averaged to one value per monitor, which matches the steady-state field the experiment reconstructs. Station longitude and latitude are projected to a local east--north frame and rescaled to the unit square.

\paragraph{Snapshots.}
Four snapshots enter the pooled analysis (Table~\ref{tab:openaq_snapshots}). They differ in day, region, and pollutant, and hence in meteorology, so each carries its own transport wind. Two further live pulls were dropped for returning fewer than the six monitors the sparsity sweep requires: a four-station Washington-metro $\mathrm{PM}_{2.5}$ pull on 2023-06-08 and a one-station Baltimore $\mathrm{PM}_{2.5}$ pull on 2024-08-10. 
A committed central-London $\mathrm{NO}_2$ snapshot (2024-07-15) with real station coordinates but synthetic concentrations ships as an offline reproducibility sample and is excluded from the pool by name.

\begin{table}[t]
  \centering
  \begin{tabular}{llcc}
    \toprule
    Region & Pollutant & Date & Monitors \\
    \midrule
    Maryland (state) & $\mathrm{PM}_{2.5}$ & 2021-07-15 & $22$ \\
    Maryland (state) & $\mathrm{PM}_{2.5}$ & 2023-07-20 & $31$ \\
    New York City    & $\mathrm{PM}_{2.5}$ & 2023-07-20 & $12$ \\
    Maryland (state) & $\mathrm{NO}_2$      & 2024-06-20 & $13$ \\
    \bottomrule
  \end{tabular}
  \caption{OpenAQ daytime snapshots in the pooled sparsity study. The monitor count is the number of live stations returned inside the bounding box after per-station aggregation.}
  \label{tab:openaq_snapshots}
\end{table}

\paragraph{Transport wind.}
The advection direction for each snapshot is derived from the Open-Meteo historical archive at the bounding-box centroid on the snapshot date, using hourly $10$~m wind speed and direction. The meteorological convention reports the direction the wind blows from, so the advection vector is $(v_E,v_N)=(-v_{\mathrm{wind}}\sin\theta,-v_{\mathrm{wind}}\cos\theta)$ for speed $v_{\mathrm{wind}}$ and bearing $\theta$. The daytime hours, $09$:$00$ to $17$:$00$ local, are combined by vector mean. Only the direction enters the operator; the magnitude is normalized away, so each snapshot contributes an advection bearing rather than a calibrated wind speed.

\paragraph{Cross-validation and pooling.}
For each snapshot, we considered sampled monitor counts
$n_{\mathrm{mon}}\in\{6,8,10,14,18,22\}$
and retained only counts no greater than the number of available stations. At each retained count, we drew \AqRealFolds{} seeded random subsets. Within each subset, $\max(2,\lfloor n_{\mathrm{mon}}/4\rfloor)$ monitors were held out for scoring and the remainder supplied the observations. Thus, $n_{\mathrm{mon}}$ includes both training and held-out monitors. The \AqRealSnaps{} snapshots yielded \AqRealCells{} snapshot--count--split evaluations.

At each monitor count, we computed ratio-of-means reductions in held-out RMSE and the sub-zero-tail diagnostic $Q_0$ defined in Section~\ref{subsec:air_quality}. Each reduction pooled the corresponding snapshot--split evaluations within the selected snapshot set. The diagnostic is calculated from posterior means and standard deviations; for the non-Gaussian constrained posterior, it is a Gaussian moment-based proxy. Bootstrap standard errors used $2{,}000$ resamples of snapshot--split rows, preserving the paired equality and constrained results within each row.

Each plotted curve uses a fixed snapshot composition. One pools all \AqRealSnaps{} snapshots at $n_{\mathrm{mon}}\in\{6,8,10\}$. The other pools the two Maryland $\mathrm{PM}_{2.5}$ snapshots over all six sampled counts. This separates changes with monitor count from changes in the snapshots contributing to a curve.

\bibliography{aaai2027}

@article{hennig2015probabilistic,
  title={Probabilistic numerics and uncertainty in computations},
  author={Hennig, Philipp and Osborne, Michael A and Girolami, Mark},
  journal={Proceedings. Mathematical, physical, and engineering sciences/the Royal Society},
  volume={471},
  number={2179},
  pages={20150142},
  year={2015}
}

@book{hennig2022probabilistic,
  title={Probabilistic Numerics: Computation as Machine Learning},
  author={Hennig, Philipp and Osborne, Michael A and Kersting, Hans P},
  year={2022},
  publisher={Cambridge University Press}
}

@article{cockayne2019bayesian,
  title={Bayesian probabilistic numerical methods},
  author={Cockayne, Jon and Oates, Chris J and Sullivan, Timothy John and Girolami, Mark},
  journal={SIAM review},
  volume={61},
  number={4},
  pages={756--789},
  year={2019},
  publisher={SIAM}
}

@book{Rasmussen2006Gaussian,
  author = {Rasmussen, Carl Edward and Williams, Christopher K. I.},
  publisher = {The MIT Press},
  title = {Gaussian Processes for Machine Learning},
  year = 2006
}

@inproceedings{graepel2003solving,
  title={Solving noisy linear operator equations by Gaussian processes: Application to ordinary and partial differential equations},
  author={Graepel, Thore},
  booktitle={ICML},
  volume={3},
  pages={234--241},
  year={2003}
}

@inproceedings{sarkka2011linear,
  title={Linear operators and stochastic partial differential equations in Gaussian process regression},
  author={S{\"a}rkk{\"a}, Simo},
  booktitle={International Conference on Artificial Neural Networks},
  pages={151--158},
  year={2011},
  organization={Springer}
}

@article{raissi2017machine,
  title={Machine learning of linear differential equations using Gaussian processes},
  author={Raissi, Maziar and Perdikaris, Paris and Karniadakis, George Em},
  journal={Journal of Computational Physics},
  volume={348},
  pages={683--693},
  year={2017},
  publisher={Elsevier}
}

@article{chen2021solving,
  title={Solving and learning nonlinear PDEs with Gaussian processes},
  author={Chen, Yifan and Hosseini, Bamdad and Owhadi, Houman and Stuart, Andrew M},
  journal={Journal of Computational Physics},
  volume={447},
  pages={110668},
  year={2021},
  publisher={Elsevier}
}

@article{pfortner2022physics,
  title={Physics-informed Gaussian process regression generalizes linear PDE solvers},
  author={Pf{\"o}rtner, Marvin and Steinwart, Ingo and Hennig, Philipp and Wenger, Jonathan},
  journal={arXiv preprint arXiv:2212.12474},
  year={2022}
}

@article{carr2005note,
  title={A note on sufficient conditions for no arbitrage},
  author={Carr, Peter and Madan, Dilip B},
  journal={Finance Research Letters},
  volume={2},
  number={3},
  pages={125--130},
  year={2005},
  publisher={Elsevier}
}

@inproceedings{cockayne2017probabilistic,
  title={Probabilistic numerical methods for PDE-constrained Bayesian inverse problems},
  author={Cockayne, Jon and Oates, Chris and Sullivan, Tim and Girolami, Mark},
  booktitle={AIP Conference Proceedings},
  volume={1853},
  number={1},
  pages={060001},
  year={2017},
  organization={AIP Publishing LLC}
}

@inproceedings{da2012gaussian,
  title={Gaussian process modeling with inequality constraints},
  author={Da Veiga, S{\'e}bastien and Marrel, Amandine},
  booktitle={Annales de la Facult{\'e} des sciences de Toulouse: Math{\'e}matiques},
  volume={21},
  number={3},
  pages={529--555},
  year={2012}
}

@article{wang2016estimating,
  title={Estimating shape constrained functions using Gaussian processes},
  author={Wang, Xiaojing and Berger, James O},
  journal={SIAM/ASA Journal on Uncertainty Quantification},
  volume={4},
  number={1},
  pages={1--25},
  year={2016},
  publisher={SIAM}
}

@misc{agrell2019gaussianprocesseslinearoperator,
      title={Gaussian processes with linear operator inequality constraints}, 
      author={Christian Agrell},
      year={2019},
      eprint={1901.03134},
      archivePrefix={arXiv},
      primaryClass={stat.ML},
      url={https://arxiv.org/abs/1901.03134}, 
}

@article{botev2017normal,
  title={The normal law under linear restrictions: simulation and estimation via minimax tilting},
  author={Botev, Zdravko I},
  journal={Journal of the Royal Statistical Society Series B: Statistical Methodology},
  volume={79},
  number={1},
  pages={125--148},
  year={2017},
  publisher={Oxford University Press}
}

@article{swiler2020survey,
  title={A survey of constrained Gaussian process regression: Approaches and implementation challenges},
  author={Swiler, Laura P and Gulian, Mamikon and Frankel, Ari L and Safta, Cosmin and Jakeman, John D},
  journal={Journal of Machine Learning for Modeling and Computing},
  volume={1},
  number={2},
  year={2020},
  publisher={Begel House Inc.}
}

@article{karniadakis2021physics,
  title={Physics-informed machine learning},
  author={Karniadakis, George Em and Kevrekidis, Ioannis G and Lu, Lu and Perdikaris, Paris and Wang, Sifan and Yang, Liu},
  journal={Nature Reviews Physics},
  volume={3},
  number={6},
  pages={422--440},
  year={2021},
  publisher={Nature Publishing Group UK London}
}

@article{jidling2017linearly,
  title={Linearly constrained Gaussian processes},
  author={Jidling, Carl and Wahlstr{\"o}m, Niklas and Wills, Adrian and Sch{\"o}n, Thomas B},
  journal={Advances in neural information processing systems},
  volume={30},
  year={2017}
}

@article{lange2018algorithmic,
  title={Algorithmic linearly constrained Gaussian processes},
  author={Lange-Hegermann, Markus},
  journal={Advances in Neural Information Processing Systems},
  volume={31},
  year={2018}
}

@inproceedings{lange2021linearly,
  title={Linearly constrained Gaussian processes with boundary conditions},
  author={Lange-Hegermann, Markus},
  booktitle={International Conference on Artificial Intelligence and Statistics},
  pages={1090--1098},
  year={2021},
  organization={PMLR}
}

@inproceedings{riihimaki2010gaussian,
  title={Gaussian processes with monotonicity information},
  author={Riihim{\"a}ki, Jaakko and Vehtari, Aki},
  booktitle={Proceedings of the thirteenth international conference on artificial intelligence and statistics},
  pages={645--652},
  year={2010},
  organization={JMLR Workshop and Conference Proceedings}
}

@article{maatouk2017gaussian,
  title={Gaussian process emulators for computer experiments with inequality constraints},
  author={Maatouk, Hassan and Bay, Xavier},
  journal={Mathematical Geosciences},
  volume={49},
  number={5},
  pages={557--582},
  year={2017},
  publisher={Springer}
}

@article{lopez2018finite,
  title={Finite-dimensional Gaussian approximation with linear inequality constraints},
  author={L{\'o}pez-Lopera, Andr{\'e}s F and Bachoc, Fran{\c{c}}ois and Durrande, Nicolas and Roustant, Olivier},
  journal={SIAM/ASA Journal on Uncertainty Quantification},
  volume={6},
  number={3},
  pages={1224--1255},
  year={2018},
  publisher={SIAM}
}

@article{bachoc2022sequential,
  title={Sequential construction and dimension reduction of Gaussian processes under inequality constraints},
  author={Bachoc, Fran{\c{c}}ois and L{\'o}pez-Lopera, Andr{\'e}s F and Roustant, Olivier},
  journal={SIAM Journal on Mathematics of Data Science},
  volume={4},
  number={2},
  pages={772--800},
  year={2022},
  publisher={SIAM}
}

@article{raissi2019physics,
  title={Physics-informed neural networks: A deep learning framework for solving forward and inverse problems involving nonlinear partial differential equations},
  author={Raissi, Maziar and Perdikaris, Paris and Karniadakis, George E},
  journal={Journal of Computational physics},
  volume={378},
  pages={686--707},
  year={2019},
  publisher={Elsevier}
}

@article{lu2021physics,
  title={Physics-informed neural networks with hard constraints for inverse design},
  author={Lu, Lu and Pestourie, Raphael and Yao, Wenjie and Wang, Zhicheng and Verdugo, Francesc and Johnson, Steven G},
  journal={SIAM Journal on Scientific Computing},
  volume={43},
  number={6},
  pages={B1105--B1132},
  year={2021},
  publisher={SIAM}
}

@article{yang2021b,
  title={B-PINNs: Bayesian physics-informed neural networks for forward and inverse PDE problems with noisy data},
  author={Yang, Liu and Meng, Xuhui and Karniadakis, George Em},
  journal={Journal of Computational Physics},
  volume={425},
  pages={109913},
  year={2021},
  publisher={Elsevier}
}

@article{dupire1994pricing,
  title={Pricing with a smile},
  author={Dupire, Bruno and others},
  journal={Risk},
  volume={7},
  number={1},
  pages={18--20},
  year={1994}
}

@article{ait2003nonparametric,
  title={Nonparametric option pricing under shape restrictions},
  author={A{\i}t-Sahalia, Yacine and Duarte, Jefferson},
  journal={Journal of Econometrics},
  volume={116},
  number={1-2},
  pages={9--47},
  year={2003},
  publisher={Elsevier}
}

@article{fengler2009arbitrage,
  title={Arbitrage-free smoothing of the implied volatility surface},
  author={Fengler, Matthias R},
  journal={Quantitative Finance},
  volume={9},
  number={4},
  pages={417--428},
  year={2009},
  publisher={Taylor \& Francis}
}

@article{cousin2016kriging,
  title={Kriging of financial term-structures},
  author={Cousin, Areski and Maatouk, Hassan and Rulli{\`e}re, Didier},
  journal={European Journal of Operational Research},
  volume={255},
  number={2},
  pages={631--648},
  year={2016},
  publisher={Elsevier}
}

@article{chataigner2021beyond,
  title={Beyond surrogate modeling: Learning the local volatility via shape constraints},
  author={Chataigner, Marc and Cousin, Areski and Cr{\'e}pey, St{\'e}phane and Dixon, Matthew and Gueye, Djibril},
  journal={SIAM Journal on Financial Mathematics},
  volume={12},
  number={3},
  pages={SC58--SC69},
  year={2021},
  publisher={SIAM}
}

@book{genz2009computation,
  title={Computation of Multivariate Normal and t Probabilities},
  author={Genz, A. and Bretz, F.},
  isbn={9783642016882},
  lccn={2009929712},
  series={Lecture Notes in Statistics},
  url={https://books.google.com/books?id=3LgGkgEACAAJ},
  year={2009},
  publisher={Springer Berlin Heidelberg}
}

@article{halton1960efficiency,
  title={On the efficiency of certain quasi-random sequences of points in evaluating multi-dimensional integrals},
  author={Halton, John H},
  journal={Numerische Mathematik},
  volume={2},
  number={1},
  pages={84--90},
  year={1960},
  publisher={Springer}
}

@article{chen2022apik,
  title={APIK: Active physics-informed kriging model with partial differential equations},
  author={Chen, Jialei and Chen, Zhehui and Zhang, Chuck and Jeff Wu, CF},
  journal={SIAM/ASA Journal on Uncertainty Quantification},
  volume={10},
  number={1},
  pages={481--506},
  year={2022},
  publisher={SIAM}
}

@article{li2024parameter,
  title={Parameter inference based on Gaussian processes informed by nonlinear partial differential equations},
  author={Li, Zhaohui and Yang, Shihao and Wu, CF Jeff},
  journal={SIAM/ASA Journal on Uncertainty Quantification},
  volume={12},
  number={3},
  pages={964--1004},
  year={2024},
  publisher={SIAM}
}

@article{golchi2015monotone,
  title={Monotone emulation of computer experiments},
  author={Golchi, Shirin and Bingham, Derek R and Chipman, Hugh and Campbell, David A},
  journal={SIAM/ASA Journal on Uncertainty Quantification},
  volume={3},
  number={1},
  pages={370--392},
  year={2015},
  publisher={SIAM}
}


\end{document}